\documentclass[10pt, a4paper]{article}
\usepackage[a4paper, margin=2.0cm]{geometry}
\usepackage[utf8]{inputenc}
\usepackage[T1]{fontenc}
\usepackage[english]{babel}
\usepackage{lmodern}
\usepackage{tikz}
\usetikzlibrary{arrows.meta, positioning, calc, backgrounds, fit, shapes.geometric}
\tikzset{needle/.style={circle, draw=blue!80!black, fill=blue!18, thick, minimum size=6.2mm, inner sep=0pt, text=blue!90!black, font=\sffamily\scriptsize\bfseries}}
\usepackage{rotating}

\tikzset{keyleaf/.style={needle, draw=blue!70!black, fill=blue!10, text=blue!80!black}}

\usepackage{amsmath, amssymb, amsthm}
\usepackage{booktabs}
\usepackage{array}
\usepackage[x11names]{xcolor}
\usepackage{hyperref}

\usepackage{booktabs}
\usepackage{pifont}
\newcommand{\cmark}{\ding{51}}
\newcommand{\xmark}{\ding{55}}

\tikzset{
    filler/.style={circle, draw=gray!60, fill=gray!12, minimum size=6.2mm, inner sep=0pt, text=gray!80},
    needle/.style={circle, draw=blue!80!black, fill=blue!18, thick, minimum size=6.2mm, inner sep=0pt, text=blue!90!black, font=\sffamily\scriptsize\bfseries},
    treenode/.style={circle, draw=teal!80!black, fill=teal!12, thick, minimum size=6.5mm, inner sep=0pt, text=teal!90!black, font=\sffamily\scriptsize\bfseries},
    rootnode/.style={circle, draw=teal!90!black, fill=teal!28, very thick, minimum size=7.8mm, inner sep=0pt, font=\sffamily\scriptsize\bfseries},
    poolnode/.style={rectangle, rounded corners=3pt, draw=purple!80!black, fill=purple!12, thick, minimum width=28mm, minimum height=6.5mm, font=\sffamily\scriptsize\bfseries},
    shift_edge/.style={-{Stealth[scale=0.75]}, draw=orange!85!black, thick, bend bottom=30},
    active_tree/.style={-{Stealth[scale=0.8]}, draw=blue!75!black, very thick},
    muted_tree/.style={-{Stealth[scale=0.7]}, draw=gray!45, dashed, thin},
    attn_mesh/.style={draw=red!35, ultra thin, opacity=0.55},
    attn_needle/.style={draw=red!75!black, thin, opacity=0.75}
}

\hypersetup{
    colorlinks=true,
    linkcolor=blue!60!black,
    citecolor=blue!60!black,
    urlcolor=blue!60!black
}

\newtheorem{proposition}{Proposition}

\title{\textbf{GTFormer: A Syntax-Free Graph Network Alternative to Transformers for 8192$\times$ Length Extrapolation in Linear Time}}
\author{
  \textbf{Elisabeth Fink} \\
  \footnotesize Middlesex University, Department of Computer Science\\ \footnotesize The Burroughs, Hendon, London NW4 4BT, UK.\\ \footnotesize \textbf{e.fink@mdx.ac.uk}}
\date{October 2026}

\begin{document}

\maketitle

\begin{abstract}
Self-attention mechanisms have established the Transformer as the dominant backbone across natural language processing, source code intelligence, and automated theorem proving. However, all-to-all pairwise token comparison enforces a quadratic $\mathcal{O}(N^2)$ compute and memory barrier alongside severe softmax attention dilution as sequence length $N$ scales beyond the training horizon. Inspired by recent advances in algebraic graph learning (\textbf{WPNet}), this work introduces \textbf{GTFormer}, a syntax-free hierarchical Graph Neural Network (GNN) that replaces flat self-attention for general long-context sequence modeling. GTFormer organizes raw token streams into a bounded-degree hierarchy combining a balanced binary reduction tree ($\lceil \log_2 N \rceil$ depth), horizontal boundary-bridging shift edges, random expander shortcuts, multi-channel convex highway gating, and a top-down broadcast pass. Across eight long-context code, log, and multi-hop retrieval benchmarks trained exclusively on sequences of length $N \le 512$, GTFormer extrapolates zero-shot to $N = 4{,}194{,}304$ ($8192\times$) on a single commodity GPU in linear $\mathcal{O}(N)$ memory, maintaining near-perfect accuracy where standard and RoPE Transformers suffer from out-of-distribution attention collapse and out-of-memory exhaustion.
\end{abstract}

\section{Introduction}

The introduction of the Transformer architecture \cite{vaswani2017attention} fundamentally reshaped modern computing, transforming the internet from an index of static documents into an interactive ecosystem of generative reasoning engines. By replacing sequential recurrence with parallelizable self-attention, Transformers unlocked unprecedented scaling laws across natural language understanding \cite{brown2020language}, automated software synthesis, biomolecular structure prediction \cite{jumper2021highly}, and formal mathematical discovery. Most notably, recent systems have demonstrated that sequence models can discover novel combinatorial algorithms \cite{romera2024mathematical} and solve Olympiad-level Euclidean geometry problems \cite{trinh2024solving}, cementing the paradigm of treating diverse symbolic, programmatic, and scientific domains as token streams processed via global context aggregation.

Despite these achievements, the foundational mechanism of the Transformer---all-to-all pairwise dot-product attention---imposes a prohibitive $\mathcal{O}(N^2)$ computational and memory bottleneck with respect to sequence length $N$. Processing ultra-long contexts such as entire software repositories, multi-day system audit logs, or genomic sequences requires storing and computing billions of pairwise attention weights, rapidly exhausting the High-Bandwidth Memory (HBM) of modern GPUs and driving unsustainable energy consumption and carbon footprints at data-center scale \cite{strubell2019energy, patterson2021carbon}. Furthermore, this hardware barrier is compounded by an intrinsic algorithmic bottleneck: as context length $N$ expands well beyond the training distribution, softmax normalization disperses attention mass across $N$ tokens ($\propto 1/N$), causing severe signal-to-noise dilution \cite{liu2024lost}, while sinusoidal and rotary positional encodings (RoPE) drift into unseen frequency regimes \cite{su2024roformer, kazemnejad2023impact}. Consequently, flat Transformers trained on short windows ($N \le 512$) fail catastrophically both in accuracy and memory footprint when deployed zero-shot on extreme context lengths ($N \ge 65,536$).

Numerous architectural alternatives have been proposed to mitigate the quadratic barrier of self-attention, yet each exhibits structural limitations when confronted with multi-hop reasoning and extreme out-of-distribution (OOD) length extrapolation \cite{tay2022efficient}. Sparse attention patterns combining local sliding windows with global or random tokens---such as Longformer \cite{beltagy2020longformer}, BigBird \cite{zaheer2020bigbird}, and Exphormer \cite{shirzad2023exphormer}---reduce per-layer complexity to $\mathcal{O}(N)$, but rely on a fixed number of stacked layers $L$; when the graph diameter of the sparse topology exceeds $L$, long-range multi-hop dependencies cannot propagate. Hierarchical attention variants such as BP-Transformer \cite{ye2019bp} and H-Transformer-1D \cite{zhu2021h} partition sequences into binary trees, yet retain flat softmax aggregation across multi-scale spans and fixed network depth. Conversely, classical Tree-LSTMs \cite{socher2011parsing, tai2015improved} require external syntactic parsers and suffer from continuous activation drift past depths of $8$ to $10$ levels. Meanwhile, modern linear State Space Models (SSMs) and recurrent architectures such as S4 \cite{gu2022efficiently} and Mamba \cite{gu2023mamba} achieve $\mathcal{O}(N)$ scaling via associative parallel scans, but are mathematically restricted to linear or diagonal state recurrences, limiting their capacity to execute non-linear relational compositions or track complex state-machine transitions across arbitrary length scales \cite{merrill2024illusion}.

The architectural solution developed in this work originates from a empirical phenomenon recently observed in computational group theory and algebraic cryptanalysis with \textbf{WPNet} \cite{fink2026learning}. In \cite{fink2026learning}, unreduced algebraic words over infinite non-abelian groups (such as the Baumslag-Solitar group $\mathrm{BS}(1,2)$ and Artin groups) were transformed from flat symbol sequences into sparse shortcut graphs and processed via message-passing Graph Neural Networks \cite{scarselli2009graph, li2016gated}. Remarkably, models trained exclusively on short algebraic words of length $N \le 1,000$ generalized zero-shot to words of length $N = 65,536$ without degradation, heuristically solving the Word Problem, predicting geodesic lengths, and breaking the Wagner-Magyarik public-key cryptosystem. Crucially, topological ablations revealed that the exact syntactic alignment of the graph edges mattered far less than the structural presence of shortcuts: by converting a chain of diameter $\mathcal{O}(N)$ into a bounded-degree graph of logarithmic diameter $\mathcal{O}(\log N)$, every node processes a small, constant fan-in neighborhood regardless of total sequence length, rendering the learned local reduction rules strictly scale-invariant.

This paper demonstrates that the shortcut-graph principle of WPNet can be generalized beyond pure mathematics to serve as a fast, low-memory, drop-in replacement for Transformers across general long-context sequence modeling tasks. The proposed architecture, \textbf{GTFormer (GraphTreeFormer)}, eliminates pairwise self-attention entirely, converting any unparsed token stream into a syntax-free multi-level graph composed of a balanced binary contraction tree of dynamic depth $D = \lceil \log_2 N \rceil$, horizontal boundary-bridging shift edges, and random expander shortcuts. To prevent activation drift across $16+$ hierarchical levels ($N = 65{,}536$ to $4{,}194{,}304$) and support indirect multi-hop lookups, GTFormer introduces a scale-invariant multi-channel convex highway gate coupled with a bidirectional top-down broadcast pass. Evaluated across an extensive suite of eight non-mathematical benchmarks---spanning chronological variable overrides, unaligned multi-token key-value bindings, 2-hop pointer dereferencing, security intrusion logs, lexical code scope shadowing, multi-needle retrieval-augmented generation (RAG), hierarchical XML tag verification, and adversarial syntactic decoys---GTFormer models trained exclusively on $N \le 512$ extrapolate zero-shot to $N = 4{,}194{,}304$ ($8192\times$ extrapolation), achieving \textbf{98.1\% $\pm$ 0.2\%} accuracy in \textbf{796.62 ms} per sequence using under \textbf{22.1 GB} of VRAM on a single Linux GPU. Figure \ref{fig:architecture_comparison} visualizes the proposed GTFormer architecture compared to a classical Transformer.

\begin{figure*}[htbp]
\centering
\begin{tikzpicture}[
    scale=0.85, transform shape,
    font=\sffamily\scriptsize
]

\begin{scope}[local bounding box=left_panel, scale=1, transform shape]
    \node[font=\sffamily\small\bfseries, align=center] at (3.85, 4.7)
        {(a)};

    \foreach \i/\lbl/\st in {
        1/t_1/filler,
        2/t_2/filler,
        3/K_r/needle,
        4/V_v/needle,
        5/t_5/filler,
        6/t_6/filler,
        7/P_p/needle,
        8/K_r/needle} 
    {
        \pgfmathsetmacro{\x}{(\i - 1) * 1.1}
        \node[\st] (L0_\i) at (\x, 0) {$\lbl$};
    }

    \node[needle] at (7.7, 0) {$K_r$};

    \foreach \i in {1,...,7} {
        \pgfmathtruncatemacro{\inext}{\i + 1}
        \ifnum\i=3
            \draw[-{Stealth[scale=0.8]}, draw=orange!90!black, very thick]
                (L0_\inext.south west) to[bend left=38]
                node[below=1pt, font=\sffamily\tiny\bfseries, text=orange!90!black] {bind odd split}
                (L0_\i.south east);
        \else
            \ifnum\i=7
                \draw[-{Stealth[scale=0.8]}, draw=orange!90!black, very thick]
                    (L0_\inext.south west) to[bend left=38]
                    node[below=1pt, font=\sffamily\tiny\bfseries, text=orange!90!black] {bind ptr}
                    (L0_\i.south east);
            \else
                \draw[-{Stealth[scale=0.65]}, draw=orange!55, semithick]
                    (L0_\inext.south west) to[bend left=32] (L0_\i.south east);
            \fi
        \fi
    }

    \node[filler]   (L1_1) at (0.55, 1.3) {$h_{1,1}$};
    \node[treenode] (L1_2) at (2.75, 1.3) {$h_{1,2}$};
    \node[filler]   (L1_3) at (4.95, 1.3) {$h_{1,3}$};
    \node[treenode] (L1_4) at (7.15, 1.3) {$h_{1,4}$};

    \draw[muted_tree]  (L0_1) -- (L1_1);
    \draw[muted_tree]  (L0_2) -- (L1_1);
    \draw[active_tree] (L0_3) -- (L1_2) node[midway, left=1pt, font=\sffamily\tiny, text=blue!80!black] {$g\!\to\!0$};
    \draw[muted_tree]  (L0_4) -- (L1_2);
    \draw[muted_tree]  (L0_5) -- (L1_3);
    \draw[muted_tree]  (L0_6) -- (L1_3);
    \draw[active_tree] (L0_7) -- (L1_4) node[midway, left=1pt, font=\sffamily\tiny, text=blue!80!black] {$g\!\to\!0$};
    \draw[muted_tree]  (L0_8) -- (L1_4);

    \node[treenode] (L2_1) at (1.65, 2.55) {$h_{2,1}$};
    \node[treenode] (L2_2) at (6.05, 2.55) {$h_{2,2}$};

    \draw[muted_tree]  (L1_1) -- (L2_1);
    \draw[active_tree] (L1_2) -- (L2_1) node[midway, right=2pt, font=\sffamily\tiny, text=blue!80!black] {$g\!\to\!1$};
    \draw[muted_tree]  (L1_3) -- (L2_2);
    \draw[active_tree] (L1_4) -- (L2_2) node[midway, right=2pt, font=\sffamily\tiny, text=blue!80!black] {$g\!\to\!1$};

    \node[rootnode] (L3_1) at (3.85, 3.8) {$h_{\text{root}}$};

    \draw[active_tree] (L2_1) -- (L3_1) node[midway, above left=-1pt, font=\sffamily\tiny, text=teal!85!black] {Ch.~1 ($K_r\!:\!V_v$)};
    \draw[active_tree] (L2_2) -- (L3_1) node[midway, above right=-1pt, font=\sffamily\tiny, text=teal!85!black] {Ch.~2 ($P_p\!\to\!K_r$)};
\end{scope}

\draw[gray!40, thick, dashed] (8.6, -0.8) -- (8.6, 5.0);

\begin{scope}[xshift=9.6cm, local bounding box=right_panel]
    \node[font=\sffamily\small\bfseries, align=center] at (3.85, 4.7)
        {(b)};

    \foreach \i/\lbl/\st in {
        1/t_1/filler,
        2/t_2/filler,
        3/K_r/needle,
        4/V_v/needle,
        5/t_5/filler,
        6/t_6/filler,
        7/P_p/needle,
        8/K_r/needle} 
    {
        \pgfmathsetmacro{\x}{(\i - 1) * 1.1}
        \node[\st] (T0_\i) at (\x, 0) {$\lbl$};
        \node[\st] (T1_\i) at (\x, 2.35) {$\lbl'$};
    }

    \node[needle] at (7.7, 0)    {$K_r$};
    \node[needle] at (7.7, 2.35) {$K_r'$};

    \begin{scope}[on background layer]
        \foreach \i in {1,...,8} {
            \foreach \j in {1,...,8} {
                \draw[attn_mesh] (T0_\i.north) -- (T1_\j.south);
            }
        }
        \foreach \j in {1,...,8} {
            \draw[attn_needle] (T0_3.north) -- (T1_\j.south);
            \draw[attn_needle] (T0_7.north) -- (T1_\j.south);
        }
    \end{scope}

    \node[fill=white, draw=red!70!black, rounded corners=2pt, inner sep=2.5pt,
          font=\sffamily\tiny\bfseries, text=red!80!black, fill opacity=0.92, text opacity=1]
          at (3.85, 1.18) {$N \times N$ Full Attention Matrix ($>1.1\text{ TB}$ VRAM at $N=131\text{k}$)};

    \node[poolnode] (POOL) at (3.85, 3.8) {Global Attention / Mean Pool};

    \foreach \i in {1,...,8} {
        \draw[-{Stealth[scale=0.65]}, draw=purple!60, thin] (T1_\i.north) -- (POOL.south);
    }

    \node[above=-1pt of POOL, font=\sffamily\tiny, text=purple!85!black]
        {Softmax weights diluted by $O(N)$ filler ($\propto 1/N$)};
\end{scope}

\end{tikzpicture}

\caption{Structural comparison of information routing over a sequence of length $N=8$ containing an odd-split key-value binding $[K_r, V_v]$ at indices $(3,4)$ and a pointer $[P_p, K_r]$ at indices $(7,8)$. \textbf{(a) GTFormer:} Local 1-hop horizontal shift edges (orange) bind adjacent tokens across odd binary boundaries before hierarchical reduction. At each tree level, shared-salience convex highway gates suppress filler nodes (dashed gray) and route active variable and pointer channels (solid blue) to the root in $\log_2 N$ steps and $O(N)$ memory. \textbf{(b) Flat Transformer:} All-to-all $O(N^2)$ self-attention connects every filler token to every other token, causing severe softmax attention dilution on out-of-distribution lengths and quadratic VRAM exhaustion.}
\label{fig:architecture_comparison}
\end{figure*}
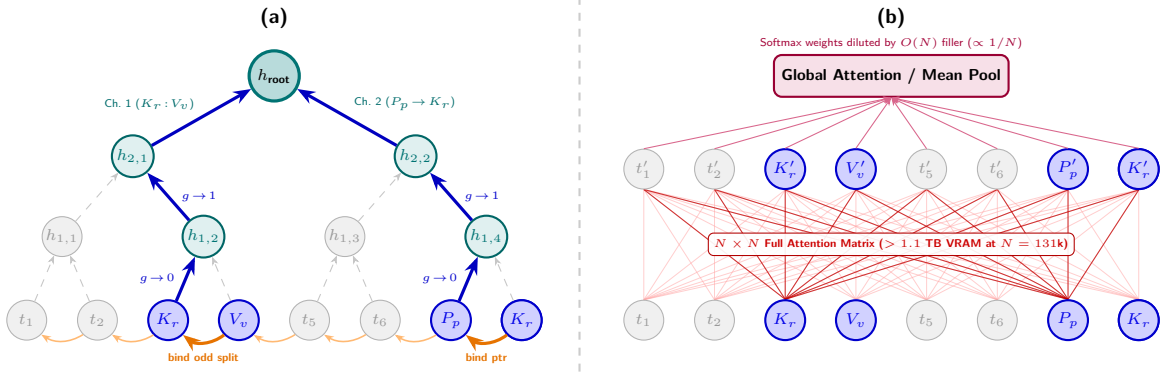

\section{Methodology}

\subsection{Long-Context Benchmark Suite}
\label{sec:benchmarks}

To evaluate whether a syntax-free graph hierarchy can replace flat self-attention without relying on an external parser, an eight-task benchmark suite is constructed to mirror the core structural challenges of source code analysis, system log auditing, and multi-hop document retrieval. Rather than training over a static, finite corpus across repeated epochs—which would permit memorization of specific token positions or document templates—the training pipeline employs an online procedural generator that synthesizes fresh document batches dynamically at every optimization step. Across $T = 2,200$ gradient steps per run, the generator cycles deterministically through six short in-distribution sequence lengths $N_{\text{train}} \in \{16, 32, 64, 128, 256, 512\}$ according to $N_t = N_{\text{train}}[t \bmod 6]$, pairing shorter sequences ($N_t \le 128$) with batch size $B = 64$ and longer sequences ($N_t \in \{256, 512\}$) with batch size $B = 32$. This schedule exposes each model to $117,344$ unique documents per training run with close to zero chance of sample reuse.

Each document $X \in \Sigma^N$ is synthesized from a vocabulary of size $|\Sigma| = 320$, partitioned into background filler tokens $\Sigma_{\text{fill}} = \{0, \dots, 199\}$ and structured operational signal tokens $\Sigma_{\text{sig}} = \{200, \dots, 319\}$. Generation begins by initializing a sequence of $N$ background tokens drawn uniformly at random from $\Sigma_{\text{fill}}$ to simulate uninformative prose or code boilerplate. The sequence is then divided into $K \in \{6, 7, 8\}$ contiguous segments of length $\lfloor N / K \rfloor$. Within each segment, task-specific signal tokens—such as register assignments, two-token key-value pairs, pointer references, scope delimiters, or syntactic decoys—are injected by overwriting filler tokens at uniformly sampled random offsets. Critically, both the intra-segment injection indices and the permutation of target registers are randomized independently for every sequence in a batch, preventing models from exploiting fixed positional heuristics or segment-to-register correlations.

The model must correctly predict the final state of every document which is governed by three latent state slots $(s_0, s_1, s_2) \in \{0, 1, 2, 3\}^3$, which map bijectively to a single 64-class exact-match target label:
\begin{equation}
    y = 16 s_0 + 4 s_1 + s_2 \in \{0, \dots, 63\},
\end{equation}
yielding a uniform random-guessing baseline of $1/64 = 1.56\%$. The eight tasks consist of: 

\begin{enumerate}
    \item[Task 1:] \textbf{Chronological Multi-Register Override (\textit{override\_qa}):} This task evaluates a model's ability to track chronological state evolution and perform order-sensitive updates across three independent registers:
\begin{itemize}
    \item \textbf{Register 0 Updates:} Encoded as integer tokens $201, 202, \text{ and } 203$, representing values $0, 1, \text{ and } 2$.
    \item \textbf{Register 1 Updates:} Encoded as integer tokens $204, 205, \text{ and } 206$, representing values $0, 1, \text{ and } 2$.
    \item \textbf{Register 2 Updates:} Encoded as integer tokens $207, 208, \text{ and } 209$, representing values $0, 1, \text{ and } 2$.
\end{itemize}

Three independent state registers ($r \in \{0, 1, 2\}$, representing logical variables such as \textit{ROLE}, \textit{CLEARANCE}, and \textit{OP}) are updated across multiple temporal waves in randomized register order. Because chronologically later assignments override earlier ones, the architecture must correctly track the final override state across sequence lengths up to $524,288$ tokens without falling victim to recency bias or memory dilution. This is equivalent to tracking state variables in command-and-control logs where later status updates override older operational clearances.

    \item[Task 2:] \textbf{Unaligned Multi-Token Associative Recall (\textit{unaligned\_kv}):} This task tests cross-boundary token binding and associative recall by deliberately unaligning key-value pairs relative to binary tree segmentation boundaries:
    
\begin{itemize}
    \item \textbf{Keys ($K_0, K_1, K_2$):} Encoded as integer tokens $200, 201, \text{ and } 202$.
    \item \textbf{Values ($V_0, V_1, V_2$):} Encoded as integer tokens $203, 204, \text{ and } 205$.
\end{itemize}

Key-value bindings are expressed as contiguous two-token phrases $\lbrack \text{KEY}_r, \text{VAL}_v \rbrack$ placed specifically at odd sequence index pairs $(2m+1, 2m+2)$ amid background standalone value distractors. In a standard binary tree hierarchy, these odd-offset positions sever the phrase—placing $\text{KEY}_r$ into the left subtree node $(2m, 2m+1)$ and $\text{VAL}_v$ into the right subtree node $(2m+2, 2m+3)$. This structural mismatch explicitly tests whether an architecture's local shortcut edges can preserve cross-boundary pairings that would otherwise be isolated by rigid partitioning. This mirrors recovering key-value metadata pairings that span across chunk boundaries in streaming data parsers where token splits disrupt standard fixed-window alignments.
    
    \item[Task 3:] \textbf{Two-Hop Pointer Dereferencing (\textit{pointer\_hop}):} The token vocabulary is structured into distinct functional ranges:
\begin{itemize}
    \item \textbf{Background Fillers:} Indices $0$ to $199$ serve as random background noise tokens.
    \item \textbf{Variables ($\text{VAR}_0$--$\text{VAR}_7$):} Encoded as integer tokens $200$ through $207$.
    \item \textbf{Values ($0$--$3$):} Encoded as integer tokens $208$ through $211$.
    \item \textbf{Pointers ($\text{PTR}_0$--$\text{PTR}_2$):} Encoded as integer tokens $212$ through $235$ (spanning 3 pointers $\times$ 8 target variables).
\end{itemize}

The sequence is partitioned into eight equal segments. Five active variables are randomly sampled from the eight candidates, assigned discrete values in $\{0, 1, 2, 3\}$, and injected into segments $0, 1, 3, 5,$ and $6$. Simultaneously, three pointers ($\text{PTR}_0, \text{PTR}_1, \text{PTR}_2$) are placed into segments $2, 4,$ and $7$. The model is tasks with finding the values of the variables that the pointers refer to. 

Crucially, because pointers are placed in segments that precede their corresponding variable definitions (e.g., $\text{PTR}_0$ appearing in segment $2$ while $\text{VAR}_2$ is defined later in segment $3$), standard causal (left-to-right) models fail. Resolving this requires bidirectional or global routing to look ahead across the sequence. This corresponds to resolving out-of-order pointer dereferences and variable references in compilers where variable usage precedes its declaration in a scope.

    \item[Task 4:] \textbf{Stateful Intrusion \& Reset Automaton (\textit{intrusion\_fsm}):} To test an architecture's capability to maintain state variables subject to intermittent global resets and subsequent updates across discrete timeline intervals, the \textit{intrusion\_fsm} task simulates a finite state machine register tracking system across seven segments. The token representation consists of:
\begin{itemize}
    \item \textbf{Background Fillers:} Indices $0$ to $199$ serve as background noise.
    \item \textbf{Register Updates:} Encoded as $200 + \text{reg} \times 4 + v$, where registers ($\text{reg}$) range from $0$ to $2$ and values ($v$) range from $0$ to $3$.
    \item \textbf{Global Reset Token:} A dedicated control token fixed at integer $250$ (\textit{reset\_tok}).
\end{itemize}

The sequence is split into seven equal segments:
\begin{enumerate}
    \item \textbf{Segments 0--2:} Three registers are initialized in a randomized order, with each assigned a value in $\{1, 2, 3\}$ at a random offset within its respective segment.
    \item \textbf{Segment 3 (Reset Check):} A 50\% chance determines whether a \textit{GLOBAL\_RESET} token ($250$) is injected. If triggered, all internal register states are wiped to $0$. Otherwise, a random fallback filler token is placed.
    \item \textbf{Segments 4--5:} Two additional registers are randomly chosen and updated with new values in $\{1, 2, 3\}$ subsequent to the potential reset event.
\end{enumerate}

This mirrors tracking protocol state machines and security session registers in firewalls that must instantly reset and re-initialize upon receiving a global TCP reset or abort command.
    
    \item[Task 5:] \textbf{Lexical Code Scope Shadowing (\textit{code\_scope\_shadow}):} This task assesses lexical scoping and variable shadowing in a simulated code execution stream. Specifically, it tests whether an architecture can maintain global variable states while temporarily processing an inner local code block, and whether it correctly executes branching memory operations—either entirely discarding local overrides (scope exit) or permanently updating global state with those inner assignments (scope commit).
    
    \begin{itemize}
    \item \textbf{Global/Local Variable Assignments:} Encoded as $200 + \text{reg} \times 4 + v$ for variables $0$ through $2$ and values $0$ through $3$.
    \item \textbf{Scope Control Tokens:} 
    \begin{itemize}
        \item \textit{SCOPE\_ENTER}: Integer $260$.
        \item \textit{SCOPE\_EXIT} (Revert): Integer $261$.
        \item \textit{SCOPE\_COMMIT}: Integer $262$.
    \end{itemize}
\end{itemize}

The sequence is divided into eight segments:
\begin{itemize}
    \item \textbf{Segments 0--2:} Global variables are assigned initial baseline values.
    \item \textbf{Segment 3:} A \textit{SCOPE\_ENTER} token ($260$) is placed.
    \item \textbf{Segments 4--6:} Inner-block local assignments temporarily shadow the global variables.
    \item \textbf{Segment 7:} A branching decision occurs with equal probability: either a \textit{SCOPE\_EXIT} token ($261$) reverts all variables back to their pre-block global values, or a \textit{SCOPE\_COMMIT} token ($262$) preserves the inner-block local overrides.
\end{itemize}

This is equivalent to tracking lexical variable scoping and block-level variable shadowing in interpreters where inner scopes temporarily override global variables before committing or reverting changes.
    
    \item[Task 6:] \textbf{Dynamic Multi-Needle RAG (\textit{multi\_needle\_rag}):} Designed to stress-test selective retrieval from a large pool of multi-entity facts based on a dynamic query token embedded midway through the text:

\begin{itemize}
    \item \textbf{Entities ($0$--$5$):} Encoded at base $200$ to $205$.
    \item \textbf{Attribute Values ($0$--$3$):} Encoded at base $206$ to $209$.
    \item \textbf{Query Combinations ($\binom{6}{3} = 20$ combinations):} Encoded via a query base starting at $210$ (tokens $210$ to $229$), where each token maps to a specific triplet of entities $(e_a, e_b, e_c)$.
\end{itemize}

The sequence is structured across seven segments:
\begin{itemize}
    \item \textbf{Segments 0, 1, 2, 4, 5, 6:} Six entity-attribute fact pairs $\lbrack \text{ENT}_e, \text{ATTR}_v \rbrack$ are scattered at random offsets.
    \item \textbf{Segment 3:} A single dynamic query token is placed right in the middle of the document, specifying a target combination of three entities to retrieve. Because the query sits in segment $3$, some target facts precede the query while others follow it, requiring bidirectional access.
\end{itemize}

This simulates retrieving specific scattered facts across a massive corporate document archive based on a dynamic user query embedded midway through the text.
    
    \item[Task 7:] \textbf{Hierarchical XML/Bracket Verification (\textit{xml\_bracket\_check}):} This task evaluates hierarchical tag matching and syntax error detection across three independent subsystems:
\begin{itemize}
    \item \textbf{Opening Tags:} Encoded starting at $200$ across 3 subsystems $\times$ 4 tag types ($200$ to $211$).
    \item \textbf{Closing Tags:} Encoded from $212$ to $223$.
\end{itemize}

The document is partitioned into six segments:
\begin{itemize}
    \item \textbf{Segments 0--2:} Subsystems $0, 1,$ and $2$ open tags of type $1$ through $3$.
    \item \textbf{Segments 3--5:} Corresponding close tags are placed. For each subsystem, there is a 50\% chance of clean closure (yielding a subsystem error status of $0$) versus a mismatched tag or noise token, which registers an unclosed/mismatched error status equal to the tag type ($1$ to $3$).
\end{itemize}

This mirrors checking syntactic correctness and tag matching across deeply nested hierarchical markup or JSON documents.
    
    \item[Task 8:] \textbf{Adversarial Syntactic Decoys (\textit{adversarial\_decoys}):} This benchmark increases sequence noise density by injecting syntax-mimicking distractor tokens around genuine key-value bindings:
    
Building directly upon the unaligned key-value base structure, this task introduces:
\begin{itemize}
    \item \textbf{Decoy Values:} Mapped from $203$ to $206$.
    \item \textbf{Decoy Keys:} Mapped from $240$ to $242$.
\end{itemize}

Genuine key-value bindings are distributed at random odd/even offsets. Simultaneously, the background filler stream is heavily seeded with high-density syntactic decoys forcing models to distinguish functional tracking states from adversarial distractors. This corresponds to parsing complex financial contracts or legal agreements that contain syntax-mimicking distractor clauses surrounding genuine contractual obligations.
\end{enumerate}

Table \ref{tab:task_examples} gives an example of the input sequence and anticipated result for each task above.

\begin{table*}[t]
\centering
\small
\caption{Representative token input sequences, event streams, and target output states across the eight benchmark tasks for sequence length $N=16$ (vocabulary size $V=320$). Filler tokens $t \in [0, 199]$ represent background noise; tokens $t \ge 200$ carry structured operational semantics.}
\label{tab:task_examples}
\begin{tabular}{p{0.18\linewidth} p{0.25\linewidth} p{0.3\linewidth} c}
\toprule
\textbf{Task Name} & \textbf{Input Token Sequence ($N=16$)} & \textbf{Event Progression / Semantics} & \textbf{Target Class} \\
\midrule
\textit{override\_qa} & 
[42, 203, 118, 207, 12, 201, 88, 206, 91, 155, 202, 33, 200, 77, 94, 18] & 
$r_0 \leftarrow 3 \to r_1 \leftarrow 3 \to r_0 \leftarrow 1 \to r_1 \leftarrow 2 \to r_2 \leftarrow 1$. \newline Final state: $[1, 2, 1]$ & 
\textbf{25} \\
\addlinespace
\textit{unaligned\_kv} & 
[42, 118, 12, \textbf{200}, \textbf{205}, 88, \textbf{201}, \textbf{203}, 91, 155, \textbf{202}, \textbf{206}, 33, \textbf{200}, \textbf{204}, 77]& 
Odd-boundary pairs: \newline $[K_0, V_2] \to [K_1, V_0] \to [K_2, V_3] \to [K_0, V_1]$. \newline Final state: $[1, 0, 3]$ & 
\textbf{19} \\
\addlinespace
\textit{pointer\_hop\_triple} & 
[14, \textbf{213}, 83, \textbf{212}, 90, \textbf{214}, 111, \textbf{200}, \textbf{211}, 49, \textbf{201}, \textbf{208}, 170, \textbf{202}, \textbf{210}, 88] & 
$V_0 \leftarrow 3$, $V_1 \leftarrow 0$, $V_2 \leftarrow 2$; \newline Dereference $\text{PTR}_0 \to V_1$, $\text{PTR}_1 \to V_0$, $\text{PTR}_2 \to V_2$ $\implies$ final values = [0, 3, 2] &
\textbf{14} \\
\addlinespace
\textit{intrusion\_fsm} & 
[81, \textbf{203}, 45, \textbf{207}, 19, \textbf{250}, 99, 110, \textbf{202}, 130, 14, 66, 92, 140, 50, 12] & 
$r_0 \leftarrow 3, r_1 \leftarrow 3 \to \text{\textbf{RESET}}(250) \implies [0, 0, 0] \to r_0 \leftarrow 2$ \newline Final state: $[2, 0, 0]$ & 
\textbf{32} \\
\addlinespace
\textit{code\_scope\_shadow} & 
[\textbf{202}, 41, \textbf{260}, 12, \textbf{203}, 89, \textbf{261}, 73, 99, 102, 18, 45, 160, 11, 82, 39] & 
Global $r_0 \leftarrow 2 \to \text{\textbf{ENTER}}(260) \to \text{Local } r_0 \leftarrow 3 \to \text{\textbf{EXIT}}(261) \implies$ revert to 2 \newline Final state: [2,0,0] & 
\textbf{32} \\
\addlinespace
\textit{multi\_needle\_rag} & 
[18, \textbf{201}, \textbf{207}, 44, \textbf{210}, 92, \textbf{200}, \textbf{206}, 11, 145, \textbf{202}, \textbf{209}, 67, 83, 104, 5] & 
Query $Q_0 (210) \implies \text{retrieve } (E_0, E_1, E_2)$; \newline Bindings: $E_1 \leftarrow 1, E_0 \leftarrow 0, E_2 \leftarrow 3$ \newline Retrieved slots: $[0, 1, 3]$ & 
\textbf{7} \\
\addlinespace
\textit{xml\_bracket\_check} & 
[33, \textbf{201}, 78, \textbf{205}, 12, 90, \textbf{213}, 14, \textbf{217}, 102, 45, 16, 88, 120, 54, 9] & 
Open tag $T_1^{(0)}$ (201) and $T_1^{(1)}$ (205); \newline Close tag $C_1^{(0)}$ (213) and $C_1^{(1)}$ (217) $\implies$ all closed ($[0, 0, 0]$) & 
\textbf{0} \\
\addlinespace
\textit{adversarial\_decoys} & 
[12, \textbf{200}, \textbf{204}, 88, \textit{240}, \textit{205}, 91, \textbf{201}, \textbf{203}, 33, \textit{204}, 77, \textbf{202}, \textbf{206}, 15, 60] & 
Real keys $[K_0, V_1], [K_1, V_0], [K_2, V_3]$ mixed with decoy $[\text{Decoy}_0 (240), V_2 (205)]$ and floating $V_1 (204)$ \newline Final state: [1,0,3] & 
\textbf{19} \\
\bottomrule
\end{tabular}
\end{table*}

\subsection{Syntax-Free GTFormer Construction from Sequential Data}
\label{sec:graph_construction}

Given an unparsed input token sequence $X = (x_0, x_1, \dots, x_{N-1}) \in \Sigma^N$ of power-of-two length $N = 2^D$ (sequences of arbitrary length are padded to the nearest power of two or reduced via strided pairs), GTFormer constructs a directed hierarchical multi-graph $\mathcal{G} = (\mathcal{V}, \mathcal{E}_{\text{shift}} \cup \mathcal{E}_{\text{tree}} \cup \mathcal{E}_{\text{exp}} \cup \mathcal{E}_{\text{down}})$ without invoking any domain-specific lexer or syntactic parser.

The vertex set $\mathcal{V} = \bigcup_{d=0}^{D} \mathcal{V}^{(d)}$ is partitioned into $D + 1$ layers, where level $d \in \{0, \dots, D\}$ contains $|\mathcal{V}^{(d)}| = N / 2^d$ nodes denoted $v_k^{(d)}$ for $k \in \{0, \dots, 2^{D-d}-1\}$. Thus, $\mathcal{V}^{(0)}$ corresponds to the $N$ leaf tokens, and $\mathcal{V}^{(D)} = \{v_0^{(D)}\}$ is the global root node. Total vertex count across the entire hierarchy is strictly linear:
\begin{equation}
|\mathcal{V}| = \sum_{d=0}^{\log_2 N} \frac{N}{2^d} = 2N - 1 = \mathcal{O}(N).
\end{equation}

Four distinct edge families govern information flow across $\mathcal{G}$, as illustrated in Figure~\ref{fig:GTFormer_topology}:
\begin{enumerate}
    \item \textbf{Horizontal Boundary-Bridging Shift Edges ($\mathcal{E}_{\text{shift}}$):} At leaf level $d = 0$, directed edges $(v_{i+1}^{(0)} \to v_i^{(0)})$ connect every adjacent token pair for $0 \le i < N - 1$. Without $\mathcal{E}_{\text{shift}}$, any two-token syntactic construct (such as $[\textit{KEY\_r}, \textit{VAL\_v}]$) starting at an odd index $2m+1$ is split between disjoint parent nodes $v_m^{(1)}$ and $v_{m+1}^{(1)}$; the horizontal shift edge guarantees that node $v_{2m+1}^{(0)}$ binds its immediate right neighbor $v_{2m+2}^{(0)}$ prior to vertical tree contraction.
    \item \textbf{Vertical Binary Tree Contraction Edges ($\mathcal{E}_{\text{tree}}$):} Between consecutive levels $d$ and $d+1$, ordered left-child and right-child edges $(v_{2k}^{(d)} \to v_k^{(d+1)})$ and $(v_{2k+1}^{(d)} \to v_k^{(d+1)})$ recursively halve the sequence resolution, compressing the effective graph diameter between any two tokens from $\mathcal{O}(N)$ to $2 \log_2 N$.
    \item \textbf{Horizontal Random Expander Shortcuts ($\mathcal{E}_{\text{exp}}$):} Within each internal level $d$ where $|\mathcal{V}^{(d)}| > 2$, a random permutation $\pi_d \in S_{|\mathcal{V}^{(d)}|}$ assigns a single long-range horizontal shortcut edge $(v_{\pi_d(k)}^{(d)} \to v_k^{(d)})$ to each node, emulating a bounded-degree expander graph \cite{shirzad2023exphormer} that creates cross-subtree bypasses.
    \item \textbf{Top-Down Broadcast Edges ($\mathcal{E}_{\text{down}}$):} After bottom-up reduction reaches the root $v_0^{(D)}$, downward broadcast edges $(v_0^{(D)} \to v_i^{(0)})$ propagate global query and pointer context back to the leaf representations to resolve multi-hop indirect references.
\end{enumerate}

Crucially, the maximum in-degree of any node in $\mathcal{G}$ is bounded by a small constant ($\Delta_{\text{in}} \le 4$: left child, right child, one horizontal peer, and self/parent), ensuring that message aggregation never scales with sequence length $N$.

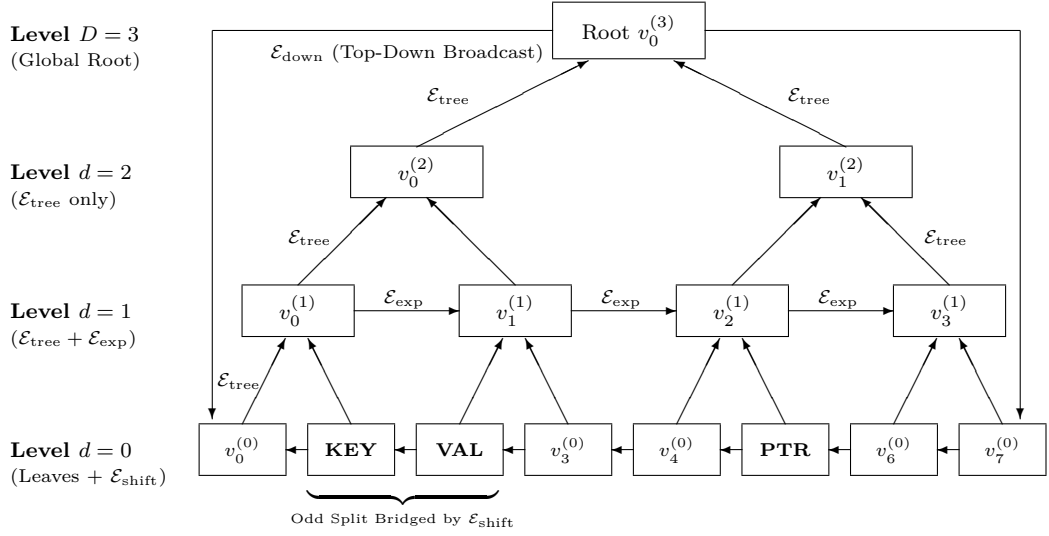
\begin{figure}[htbp]
\centering
\setlength{\unitlength}{0.95pt}
\begin{picture}(460, 215)(0, 0)
    \put(0, 182){\textbf{\footnotesize Level $D=3$}}
    \put(0, 172){\scriptsize (Global Root)}
    \put(0, 127){\textbf{\footnotesize Level $d=2$}}
    \put(0, 117){\scriptsize ($\mathcal{E}_{\text{tree}}$ only)}
    \put(0, 72){\textbf{\footnotesize Level $d=1$}}
    \put(0, 62){\scriptsize ($\mathcal{E}_{\text{tree}} + \mathcal{E}_{\text{exp}}$)}
    \put(0, 17){\textbf{\footnotesize Level $d=0$}}
    \put(0, 7){\scriptsize (Leaves + $\mathcal{E}_{\text{shift}}$)}

    \put(215, 175){\framebox(60, 22){\footnotesize Root $v_0^{(3)}$}}

    \put(215, 186){\line(-1, 0){135}}
    \put(80, 186){\vector(0, -1){154}}
    \put(103, 175){\scriptsize $\mathcal{E}_{\text{down}}$ (Top-Down Broadcast)}
    \put(275, 186){\line(1, 0){125}}
    \put(400, 186){\vector(0, -1){154}}

    \put(135, 120){\framebox(52, 20){\footnotesize $v_0^{(2)}$}}
    \put(305, 120){\framebox(52, 20){\footnotesize $v_1^{(2)}$}}
    \put(161, 140){\vector(2, 1){68}}
    \put(331, 140){\vector(-2, 1){68}}
    \put(165, 158){\scriptsize $\mathcal{E}_{\text{tree}}$}
    \put(308, 158){\scriptsize $\mathcal{E}_{\text{tree}}$}

    \put(92, 65){\framebox(44, 20){\footnotesize $v_0^{(1)}$}}
    \put(178, 65){\framebox(44, 20){\footnotesize $v_1^{(1)}$}}
    \put(264, 65){\framebox(44, 20){\footnotesize $v_2^{(1)}$}}
    \put(350, 65){\framebox(44, 20){\footnotesize $v_3^{(1)}$}}
    \put(136, 75){\vector(1, 0){42}}
    \put(222, 75){\vector(1, 0){42}}
    \put(308, 75){\vector(1, 0){42}}
    \put(148, 78){\scriptsize $\mathcal{E}_{\text{exp}}$}
    \put(234, 78){\scriptsize $\mathcal{E}_{\text{exp}}$}
    \put(320, 78){\scriptsize $\mathcal{E}_{\text{exp}}$}
    \put(114, 85){\vector(1, 1){35}}
    \put(200, 85){\vector(-1, 1){35}}
    \put(286, 85){\vector(1, 1){35}}
    \put(372, 85){\vector(-1, 1){35}}
    \put(110, 102){\scriptsize $\mathcal{E}_{\text{tree}}$}
    \put(362, 102){\scriptsize $\mathcal{E}_{\text{tree}}$}

    \put(75, 10){\framebox(34, 20){\scriptsize $v_0^{(0)}$}}
    \put(118, 10){\framebox(34, 20){\scriptsize \textbf{KEY}}}
    \put(161, 10){\framebox(34, 20){\scriptsize \textbf{VAL}}}
    \put(204, 10){\framebox(34, 20){\scriptsize $v_3^{(0)}$}}
    \put(247, 10){\framebox(34, 20){\scriptsize $v_4^{(0)}$}}
    \put(290, 10){\framebox(34, 20){\scriptsize \textbf{PTR}}}
    \put(333, 10){\framebox(34, 20){\scriptsize $v_6^{(0)}$}}
    \put(376, 10){\framebox(34, 20){\scriptsize $v_7^{(0)}$}}

    \put(118, 20){\vector(-1, 0){9}}
    \put(161, 20){\vector(-1, 0){9}}
    \put(204, 20){\vector(-1, 0){9}}
    \put(247, 20){\vector(-1, 0){9}}
    \put(290, 20){\vector(-1, 0){9}}
    \put(333, 20){\vector(-1, 0){9}}
    \put(376, 20){\vector(-1, 0){9}}
    \put(111, 5){\scriptsize $\underbrace{\hspace{72pt}}_{\text{Odd Split Bridged by }\mathcal{E}_{\text{shift}}}$}

    \put(92, 30){\vector(1, 2){17}}
    \put(135, 30){\vector(-1, 2){17}}
    \put(178, 30){\vector(1, 2){17}}
    \put(221, 30){\vector(-1, 2){17}}
    \put(264, 30){\vector(1, 2){17}}
    \put(307, 30){\vector(-1, 2){17}}
    \put(350, 30){\vector(1, 2){17}}
    \put(393, 30){\vector(-1, 2){17}}
    \put(82, 45){\scriptsize $\mathcal{E}_{\text{tree}}$}
\end{picture}
\caption{Architecture of the syntax-free \textbf{GTFormer} topology shown for $N=8$ ($D=3$). At Level $d=0$, an unaligned two-token phrase $[\textit{KEY}, \textit{VAL}]$ at odd indices $(v_1^{(0)}, v_2^{(0)})$ straddles the boundary between subtrees $v_0^{(1)}$ and $v_1^{(1)}$, but is bound locally via the horizontal shift edge $\mathcal{E}_{\text{shift}}$ ($v_2^{(0)} \to v_1^{(0)}$). Bottom-up tree edges $\mathcal{E}_{\text{tree}}$ and horizontal expander shortcuts $\mathcal{E}_{\text{exp}}$ compress sequence state in $\log_2 N$ steps, while top-down broadcast edges $\mathcal{E}_{\text{down}}$ resolve remote pointer queries ($\textit{PTR}$).}
\label{fig:GTFormer_topology}
\end{figure}

\subsection{The GTFormer Neural Architecture}
\label{sec:architecture}

The token embedding function $e : \Sigma \to \mathbb{R}^{d_{\text{model}}}$ maps each discrete token from the vocabulary $\Sigma$ to a continuous vector of dimension $d_{\text{model}}$, providing the initial continuous features for the leaf nodes $v_i^{(0)}$ at the base of the graph hierarchy. These features will be learned by the model during training. GTFormer processes $\mathcal{G}$ in three sequential stages using shared weights in the neural operators across all tree levels $d \in \{0, \dots, D-1\}$

\paragraph{Stage 1: Local Boundary-Bridging Encoding ($\mathcal{E}_{\text{shift}}$).}
Each leaf token $x_i$ aggregates its right neighbor $x_{i+1}$ (with zero-padding at $i = N - 1$) through a two-layer GELU Multi-Layer Perceptron with pre-layer normalization:
\begin{equation}
h_i^{(0)} = u_i = W_2 \phi\Big(W_1 \mathrm{LN}\big([e(x_i) \,\|\, e(x_{i+1})]\big) + b_1\Big) + b_2 \in \mathbb{R}^{d_{\text{model}}},
\label{eq:local_shift}
\end{equation}
where $\|$ denotes vector concatenation and $\phi(\cdot)$ is the GELU activation.

\paragraph{Stage 2: Bottom-Up Expander-Augmented Convex Tree Reduction ($\mathcal{E}_{\text{exp}}, \mathcal{E}_{\text{tree}}$).}
At each level $d \in \{0, \dots, D-1\}$, nodes first exchange messages along horizontal random expander edges $\mathcal{E}_{\text{exp}}$ controlled by a learned sigmoid gate initialized with a conservative negative bias ($b_{\text{exp}} = -4.0$, so $\sigma(-4.0) \approx 0.018$):
\begin{equation}
\tilde{h}_j^{(d)} = h_j^{(d)} + \sigma\Big(W_{\text{exp}}\big[h_j^{(d)} \,\|\, h_{\pi_d(j)}^{(d)}\big] + b_{\text{exp}}\Big) \odot \big(V_{\text{exp}} h_{\pi_d(j)}^{(d)}\big).
\label{eq:expander_step}
\end{equation}
Next, sibling pairs $\big(\tilde{h}_{2k}^{(d)}, \tilde{h}_{2k+1}^{(d)}\big)$ contract into parent node $h_k^{(d+1)}$ via a \textbf{multi-channel convex highway gate}. Specifically, a channel-wise gating vector $\alpha_k^{(d)} \in (0, 1)^{d_{\text{model}}}$ is computed from the concatenated sibling states:
\begin{equation}
    \alpha_k^{(d)} = \sigma\!\Big(W_g \big[\tilde{h}_{2k}^{(d)} \,\|\, \tilde{h}_{2k+1}^{(d)}\big] + b_g\Big),
\end{equation}
where $\sigma(\cdot)$ denotes the element-wise sigmoid activation. The parent representation is then formed as a channel-wise convex combination:
\begin{equation}
    h_k^{(d+1)} = \alpha_k^{(d)} \odot \tilde{h}_{2k}^{(d)} + \big(1 - \alpha_k^{(d)}\big) \odot \tilde{h}_{2k+1}^{(d)},
\end{equation}
where $\odot$ denotes the Hadamard (element-wise) product. Because the coefficients per channel are strictly non-negative and sum up to $1$, the contraction stays within the convex hull of the child representations. This bounds activation magnitudes across tree levels $d = 0, \dots, D-1$ and enables unimpeded gradient and feature routing from informative leaf tokens while suppressing background filler tokens. Rather than applying an unconstrained MLP across levels---which causes activation norms and feature manifolds to drift exponentially when unrolled to $D = 17$ levels ($N = 524,288$)---GTFormer computes per-channel salience scores $s_L, s_R : \mathbb{R}^{d_{\text{model}}} \to \mathbb{R}^{d_{\text{model}}}$ and a cross-child modulation term:
\begin{align}
g_k^{(d)} &= \sigma\left( \frac{s_R\big(\tilde{h}_{2k+1}^{(d)}\big) - s_L\big(\tilde{h}_{2k}^{(d)}\big)}{\sqrt{d_{\text{model}}}} + \tanh\Big(W_{\text{cross}}\big[\tilde{h}_{2k}^{(d)} \,\|\, \tilde{h}_{2k+1}^{(d)}\big]\Big) + b_{\text{override}} \right), \label{eq:convex_gate} \\
h_k^{(d+1)} &= g_k^{(d)} \odot \tilde{h}_{2k+1}^{(d)} + \big(\mathbf{1} - g_k^{(d)}\big) \odot \tilde{h}_{2k}^{(d)}, \label{eq:convex_merge}
\end{align}
where $b_{\text{override}} \in \mathbb{R}^{d_{\text{model}}}$ is initialized to $+0.2$ to provide a chronological recency prior favoring the right (later) subtree when salience is tied. First, within level $d$, horizontal message passing over $\mathcal{E}_{\text{shift}}$ (for $d=0$) or the expander shortcuts $\mathcal{E}_{\text{exp}}$ (for $d > 0$) updates each node state $h_i^{(d)} \mapsto \tilde{h}_i^{(d)}$, enabling direct cross-subtree communication before vertical compression. Next, horizontally enriched sibling pairs $\big(\tilde{h}_{2k}^{(d)}, \tilde{h}_{2k+1}^{(d)}\big)$ contract into parent node $h_k^{(d+1)}$ via a \textbf{multi-channel convex highway gate}. Specifically, a channel-wise gating vector $\alpha_k^{(d)} \in (0, 1)^{d_{\text{model}}}$ is computed from the concatenated sibling states:
\begin{equation}
    \alpha_k^{(d)} = \sigma\!\Big(W_g \big[\tilde{h}_{2k}^{(d)} \,\|\, \tilde{h}_{2k+1}^{(d)}\big] + b_g\Big),
\end{equation}
where $\sigma(\cdot)$ denotes the element-wise sigmoid activation. The parent representation is then formed as a channel-wise convex combination:
\begin{equation}
    h_k^{(d+1)} = \alpha_k^{(d)} \odot \tilde{h}_{2k}^{(d)} + \big(1 - \alpha_k^{(d)}\big) \odot \tilde{h}_{2k+1}^{(d)},
\end{equation}
where $\odot$ denotes the Hadamard (element-wise) product. While the convex gate preserves signal norm and selectively filters background noise during vertical reduction, the preceding expander shortcuts $\mathcal{E}_{\text{exp}}$ encoded in $\tilde{h}^{(d)}$ prevent the root node $v_0^{(D)}$ from becoming an information bottleneck when routing multiple concurrent variable bindings.

\begin{proposition}[Manifold Boundedness Across Arbitrary Extrapolation Depth]
\label{prop:manifold}
Consider the vertical tree contraction in Equation~\eqref{eq:convex_merge} with expander gates closed ($g_{\text{exp}} = 0$). For any sequence length $N = 2^D$ and any feature channel $c \in \{1, \dots, d_{\text{model}}\}$, the root representation $h_{0,c}^{(D)}$ lies strictly within the convex hull of the leaf representations $\{u_{i,c}\}_{i=0}^{N-1}$:
\begin{equation}
\min_{0 \le i < N} u_{i,c} \;\le\; h_{0,c}^{(D)} \;\le\; \max_{0 \le i < N} u_{i,c} \quad \forall D \ge 1.
\end{equation}
\end{proposition}
\begin{proof}
Since $g_{k,c}^{(d)} \in (0, 1)$, each parent coordinate $h_{k,c}^{(d+1)}$ is a strict convex combination of $\tilde{h}_{2k,c}^{(d)}$ and $\tilde{h}_{2k+1,c}^{(d)}$. By induction from $d=0$ to $D$, $h_{0,c}^{(D)} = \sum_{i=0}^{N-1} w_{i,c} u_{i,c}$ with $w_{i,c} > 0$ and $\sum_{i=0}^{N-1} w_{i,c} = 1$, preventing norm explosion or out-of-manifold drift at arbitrary tree depth $D$.
\end{proof}

\paragraph{Stage 3: Top-Down Broadcast Pass and Multi-Head Readout ($\mathcal{E}_{\text{down}}$).}
To resolve multi-hop pointer dereferencing and dynamic RAG queries, the global root state $h_0^{(D)}$ is broadcast back down to the leaf representations $u_i$ via multiplicative bilinear conditioning:
\begin{align}
z_i &= u_i \odot \big(W_{\text{td}} h_0^{(D)}\big), \quad \gamma_i = \sigma\Big(\mathrm{MLP}_{\text{gate}}\big([u_i \,\|\, W_{\text{td}} h_0^{(D)} \,\|\, z_i]\big)\Big), \\
\hat{u}_i &= u_i + \gamma_i \odot \mathrm{MLP}_{\text{upd}}\big([u_i \,\|\, W_{\text{td}} h_0^{(D)} \,\|\, z_i]\big), \quad
h_{\text{td}} = \sum_{i=0}^{N-1} \frac{\exp(W_q \hat{u}_i)}{\sum_{j=0}^{N-1} \exp(W_q \hat{u}_j)} \odot \hat{u}_i. \label{eq:top_down_pool}
\end{align}
At extreme context lengths ($N > 16,384$), Equation~\eqref{eq:top_down_pool} is evaluated via an exact two-pass streaming max-stabilized softmax over chunks of size $C = 16,384$, bounding peak intermediate memory to $\mathcal{O}(C \cdot d_{\text{model}})$. Finally, the concatenated root and top-down summary vector $h_{\text{final}} = W_{\text{merge}}[h_0^{(D)} \,\|\, h_{\text{td}}]$ is passed to a shared readout MLP producing 64-class exact-match logits $\hat{\ell}_{\text{joint}} \in \mathbb{R}^{64}$ alongside three auxiliary 4-way slot logits $\hat{\ell}_{\text{slot}}^{(m)} \in \mathbb{R}^4$ ($m \in \{0,1,2\}$). To eliminate permutation symmetry and training saddle points on relational tasks (e.g., pointer dereferencing and multi-needle RAG), the training loss combines three distinct objectives:
\begin{equation}
\mathcal{L}_{\text{total}} = \mathcal{L}_{\text{joint}} + \lambda_1 \mathcal{L}_{\text{slots}} + \lambda_2 (\mathcal{L}_{\text{vars}} + \mathcal{L}_{\text{ptrs}})
\end{equation}
where:
\begin{itemize}
    \item \textbf{$\mathcal{L}_{\text{joint}}$}: Cross-entropy over the unified 64-class exact-match target space ($4^3$).
    \item \textbf{$\mathcal{L}_{\text{slots}}$}: Factored auxiliary cross-entropy across individual output slots.
    \item \textbf{$\mathcal{L}_{\text{vars}}$ and $\mathcal{L}_{\text{ptrs}}$}: Direct supervision for internal variables and pointers, ensuring the model learns correct multi-hop routing without getting confused by token order.
\end{itemize}

\vspace{0.3cm}

\textbf{Architecture Ablation:} To isolate the contribution of each component in GTFormer, we evaluate seven distinct architectures, ranging from fully-featured variants to structural ablations and standard flat baselines as shown in Table \ref{tab:model_ablations}.

\begin{table}[htbp]
\centering
\caption{Architectural comparison of evaluated model variants and ablation configurations.}
\label{tab:model_ablations}
\resizebox{\textwidth}{!}{%
\begin{tabular}{@{}llccccp{5.2cm}@{}}
\toprule
\textbf{Model Variant} & \textbf{Topology} & \textbf{Local Shift} & \textbf{Expander} & \textbf{Convex Gate} & \textbf{Top-Down} & \textbf{Primary Role / Ablation Target} \\
 & & (\textit{use\_local\_shift}) & (\textit{use\_expander}) & (\textit{use\_convex\_gate}) & (\textit{use\_top\_down}) & \\
\midrule
\textit{GTFormer\_Bidirectional} & Hierarchical & \cmark & \cmark & \cmark & \cmark & Full proposed architecture \\
\textit{GTFormer\_BottomUp}      & Hierarchical & \cmark & \cmark & \cmark & \xmark & Feed-forward root readout only \\
\textit{GTFormer\_NoExpander}    & Hierarchical & \cmark & \xmark & \cmark & \cmark & Ablates horizontal peer shortcuts ($\mathcal{E}_{\text{exp}}$) \\
\textit{Pure\_Binary\_Tree}             & Hierarchical & \xmark & \xmark & \cmark & \cmark & No horizontal edges; tests odd-index splitting \\
\textit{Unbounded\_Tree\_MLP}           & Hierarchical & \cmark & \cmark & \xmark & \cmark & Unconstrained MLP merge; tests depth drift \\
\midrule
\textit{Flat\_Transformer\_Sinusoidal}  & Flat (2L)    & ---    & ---    & ---    & ---    & Sinusoidal PE + mean pooling \\
\textit{Flat\_Transformer\_RoPE}        & Flat (2L)    & ---    & ---    & ---    & ---    & Rotary PE (RoPE) + attention pooling \\
\bottomrule
\end{tabular}%
}
\end{table}

Models are optimized using AdamW \cite{loshchilov2019adamw} with a weight decay of $10^{-4}$, paired with a cosine annealing learning rate scheduler \cite{loshchilov2016sgdr} where the learning rate decays over $T_{\text{max}}$ steps to a minimum floor of $5\%$ of its initial value.

\subsection{Evaluation Metrics and Complexity Comparison}
\label{sec:metrics}

At evaluation time, the auxiliary slot heads used for intermediate supervision are removed. Performance is assessed strictly on the joint 64-class exact-match objective, $\hat{y} = \arg\max_{c \in \{0,\dots,63\}} \hat{\ell}_{\text{joint}, c}$, across a selection of sequence lengths of $N=2^i, i\in\{6,8,9,10,12,14,16,17,18,19,20,21,22\}$ using four quantitative metrics:
\begin{enumerate}
    \item \textbf{Zero-Shot Exact-Match Accuracy ($\mathrm{Acc}(N)$):} Percentage of test documents at context length $N$ where $\hat{y} = y$ (random chance $= 1.56\%$). Lengths $N \in \{64, 256, 512\}$ measure in-distribution (ID) generalization, whereas $N \in \{1024, \dots, 4194304\}$ measure zero-shot out-of-distribution (OOD) extrapolation up to $8192\times$ the maximum training length.
    \item \textbf{Extrapolation Retention Ratio ($\mathrm{ERR}_{N}$):} Defined as $\mathrm{ERR}_{N} = \mathrm{Acc}(N) / \mathrm{Acc}(512)$, quantifying the fraction of in-distribution accuracy preserved at extreme OOD scale ($N = 65{,}536$ and $N = 4{,}194{,}304$).
    \item \textbf{Wall-Clock Inference Latency ($\tau(N)$):} Mean execution time in milliseconds per sequence ($\mathrm{ms/seq}$), measured with hardware-synchronized CUDA timers across batched evaluation passes.
    \item \textbf{Peak GPU Memory Footprint ($M_{\text{peak}}(N)$):} Maximum allocated CUDA memory in megabytes ($\mathrm{MB}$) recorded via \textit{torch.cuda.max\_memory\_allocated()} at each context length $N$.
\end{enumerate}

Table~\ref{tab:complexity_comparison} contrasts the theoretical and structural properties of GTFormer against flat Transformers and tree-structured baselines.

\begin{table}[htbp]
\centering
\small
\setlength{\tabcolsep}{3.5pt} 
\begin{tabular}{p{5cm}cccccc}
\toprule
\textbf{Architecture} & \textbf{Time} & \textbf{Memory} & \textbf{Graph} & \textbf{Max} & \textbf{Odd-Split} & \textbf{2-Hop Ptr} \\
& & & \textbf{Diam.}  & \textbf{Fan-In} & \textbf{Safe} & \textbf{Safe} \\
\midrule
Flat Transformer (Sin / RoPE) & $\mathcal{O}(N^2)$ & $\mathcal{O}(N^2)$ & $1$ & $N$ & Yes & Yes (ID) \\[3pt]
Pure Binary Tree-GNN & $\mathcal{O}(N)$ & $\mathcal{O}(N)$ & $2 \log N$ & $2$ & No & No \\[2pt]
Unbounded Tree-MLP & $\mathcal{O}(N)$ & $\mathcal{O}(N)$ & $2 \log N$ & $3$ & Yes & No \\[2pt]
GraphTree (Bottom-Up) & $\mathcal{O}(N)$ & $\mathcal{O}(N)$ & $\le 2 \log N$ & $4$ & Yes & Partial \\[2pt]
\textbf{GraphTree (Bidirectional)} & $\mathbf{\mathcal{O}(N)}$ & $\mathbf{\mathcal{O}(N)}$ & $\mathbf{\le 2 \log N}$ & $\mathbf{4}$ & \textbf{Yes} & \textbf{Yes ($1024\times$)} \\
\bottomrule
\end{tabular}
\caption{Theoretical complexity and structural comparison across sequence length $N$. Bounded fan-in ($\Delta_{\text{in}} \le 4$) and convex highway gating enable scalable linear extrapolation to $N = 4{,}194{,}304$ without attention dilution.}
\label{tab:complexity_comparison}
\end{table}

To account for the quadratic $O(N^2)$ memory complexity inherent to full self-attention, flat transformer baselines are evaluated up to a maximum sequence length threshold of $N = 4,096$. Sequence lengths exceeding this limit ($N > 4,096$) exhaust standard single-GPU VRAM capacities and are systematically recorded as Out-Of-Memory (OOM) to prevent crashing the evaluation suite. Due to the quadratic $O(N^2)$ scaling of full self-attention, the VRAM required to store attention score matrices increases exponentially with sequence length. Based on the evaluation batch sizes and two-layer transformer encoder setup, required memory exceeds standard single-GPU capacity at extended lengths: approximately $68.7\text{ GB}$ at $N = 16{,}384$ (batch size $8$), $550\text{ GB}$ at $N = 65{,}536$ (batch size $4$), and over $1.1\text{ TB}$ at $N = 131{,}072$ (batch size $2$).

\section{Computational Aspects}

Experiments were executed using Python 3.11.2, PyTorch 2.10.0+cu128, and CUDA 12.2 on an NVIDIA Quadro RTX 8000 GPU (48 GB VRAM), with memory usage explicitly capped at 4 GB to simulate real-world hardware constraints.

\section{Results}

\subsection{Aggregate Performance Across Sequence Lengths, Training Duration, and Memory Usage}

Models are trained exclusively on short sequence lengths ($N \in \{32, 64, 128, 256, 512\}$) for $2{,}200$ optimization steps and evaluated zero-shot on out-of-distribution (OOD) sequences scaling up to $N = 4{,}194{,}304$ ($8192\times$ extrapolation beyond the maximum training horizon). The best performing architecture was \textit{GraphTree (Local Shift + No Expander)} and achieved a striking $98.1\pm3\%$ accuracy when evaluated on sequences of length $N=2^{22}$. The evaluation was stopped at that sequence length due to memory constraints. Figure~\ref{fig:all_tests} illustrates the exact-match accuracy trajectories across all seven sequence lengths, and Table~\ref{tab:aggregate_summary} summarizes the cross-task average accuracy, per-sequence inference latency, and peak GPU memory consumption for all seven evaluated architectures.

\begin{figure}[htbp]
    \centering
    \includegraphics[width=1\linewidth]{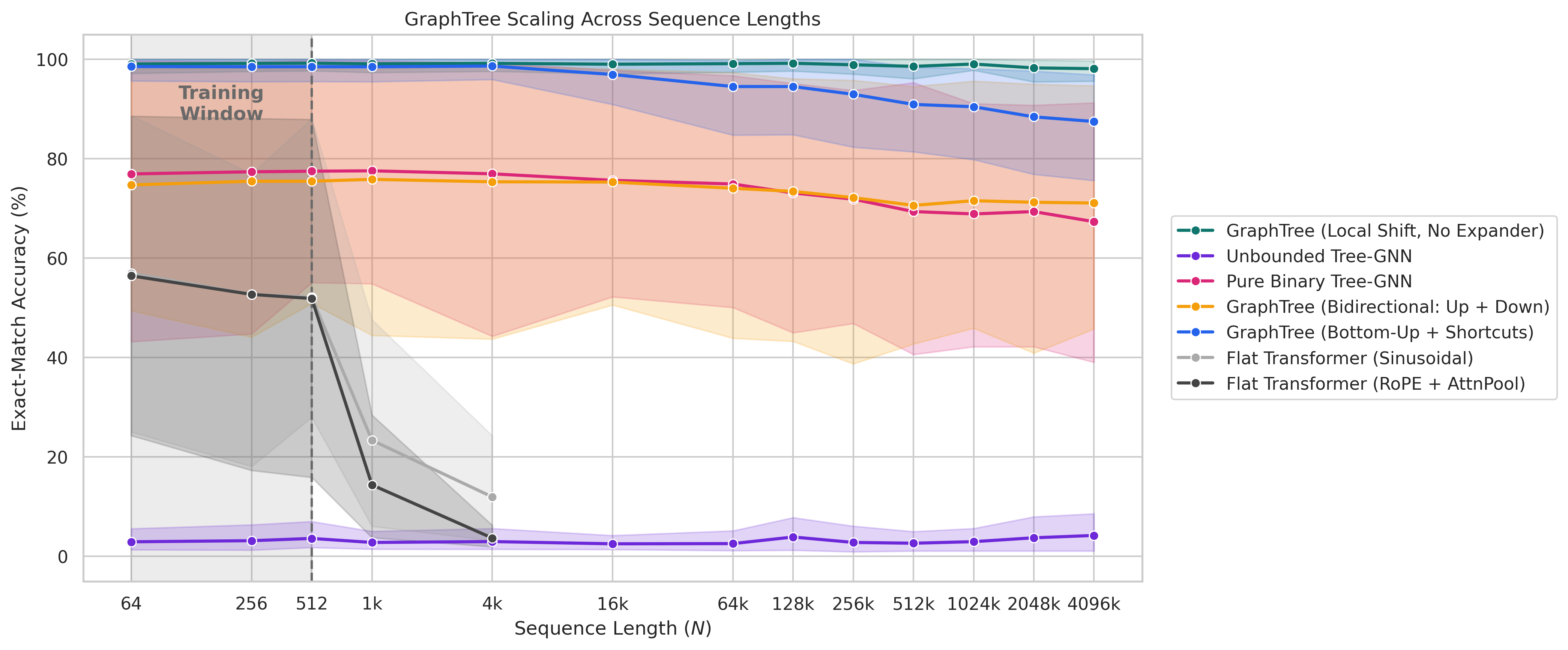}
    \caption{Average performance on all tests in zero-shot extrapolation of all architectures.}
    \label{fig:all_tests}
\end{figure}

\begin{table*}[htbp]
\centering
\tiny
\setlength{\tabcolsep}{2.5pt}

\label{tab:overall_summary_part1}
\begin{tabular}{p{3.2cm}cccccc}
\toprule
Architecture & $N=64$ (ID) & $N=256$ (ID) & $N=512$ (ID) & $N=1024$ & $N=4096$ & $N=16384$ \\
\midrule
GraphTree (Bidirectional: Up + Down) & $70.8\%$ & $72.0\%$ & $72.1\%$ & $72.2\%$ & $71.8\%$ & $71.2\%$ \\
GraphTree (Bottom-Up + Shortcuts) & $97.3\%$ & $97.3\%$ & $97.3\%$ & $97.3\%$ & $97.3\%$ & $95.5\%$ \\
GraphTree (Local Shift, No Expander) & $97.8\%$ & $97.9\%$ & $98.0\%$ & $97.8\%$ & $97.9\%$ & $97.7\%$ \\
Pure Binary Tree-GNN (No Shortcuts) & $86.7\%$ & $86.5\%$ & $86.9\%$ & $86.9\%$ & $85.9\%$ & $83.3\%$ \\
Unbounded Tree-GNN (No Convex Gate) & $3.0\%$ & $3.1\%$ & $3.5\%$ & $3.1\%$ & $3.1\%$ & $2.8\%$ \\
Flat Transformer (Sinusoidal) & $59.5\%$ & $54.3\%$ & $53.9\%$ & $32.4\%$ & $16.0\%$ & OOM \\
Flat Transformer (RoPE + AttnPool) & $57.9\%$ & $54.2\%$ & $53.3\%$ & $25.2\%$ & $5.5\%$ & OOM \\
\bottomrule
\end{tabular}

\vspace{1.5em}

\begin{tabular}{p{3.2cm}ccccccc}
\toprule
Architecture & $N=65536$ & $N=131072$ & $N=262144$ & $N=524288$ & $N=1048576$ & $N=2097152$ & $N=4194304$ \\
\midrule
GraphTree (Bidirectional: Up + Down) & $69.0\%$ & $68.4\%$ & $66.4\%$ & $63.9\%$ & $65.8\%$ & $65.1\%$ & $71.1\%$ \\
GraphTree (Bottom-Up + Shortcuts) & $93.0\%$ & $92.6\%$ & $90.5\%$ & $88.1\%$ & $88.1\%$ & $86.5\%$ & $87.5\%$ \\
GraphTree (Local Shift, No Expander) & $97.9\%$ & $98.0\%$ & $97.2\%$ & $96.5\%$ & $96.9\%$ & $95.9\%$ & $98.1\%$ \\
Pure Binary Tree-GNN (No Shortcuts) & $81.7\%$ & $78.1\%$ & $75.6\%$ & $71.1\%$ & $68.9\%$ & $68.5\%$ & $67.3\%$ \\
Unbounded Tree-GNN (No Convex Gate) & $2.7\%$ & $3.5\%$ & $2.6\%$ & $2.5\%$ & $2.0\%$ & $3.0\%$ & $4.2\%$ \\
Flat Transformer (Sinusoidal) & OOM & OOM & OOM & OOM & OOM & OOM & OOM \\
Flat Transformer (RoPE + AttnPool) & OOM & OOM & OOM & OOM & OOM & OOM & OOM \\
\bottomrule
\end{tabular}
\label{tab:aggregate_summary}
\caption{Aggregate summary of results over all eight tasks (OOM=Out-Of-Memory).}
\end{table*}

Table~\ref{tab:aggregate_summary} reports accuracy averaged over all eight tasks, with
$N \in \{64, 256, 512\}$ in-distribution and all longer sequences out-of-distribution.
Both flat Transformer baselines degrade rapidly with sequence length: accuracy drops from
$57.9$--$59.5\%$ at $N{=}64$ to $25.2$--$32.4\%$ at $N{=}1024$ and to $5.5$--$16.0\%$ at
$N{=}4096$, and both run out of memory from $N{=}16384$ onward, consistent with the quadratic
cost of full attention. In contrast, every tree-based model runs up to $N{=}4{,}194{,}304$,
i.e., $8{,}192\times$ the longest training length of $N=512$.

Among the tree-based models, the ablations isolate the contribution of each component.
The Pure Binary Tree-GNN reaches $86.7\%$ at $N{=}64$ but degrades steadily with length, to
$67.3\%$ at $N{=}4\mathrm{M}$, suggesting that without shortcuts information must traverse
increasingly long tree paths. Adding shortcuts (Bottom-Up + Shortcuts) raises ID accuracy to
$97.3\%$, which remains constant through $N{=}4096$ and drops only moderately at larger lengths
($95.5\%$ at $N{=}16384$, $87.5\%$ at $N{=}4\mathrm{M}$). The Local Shift variant without the
expander is the strongest and most stable configuration, with accuracy between $95.9\%$ and
$98.1\%$ at every length. The Bidirectional (Up + Down) variant is substantially weaker
($63.9$--$72.2\%$), indicating that the additional top-down pathway does not help on these
tasks. Finally, removing the convex gate is catastrophic: the Unbounded Tree-GNN remains at
$2.0$--$4.2\%$ at all lengths, showing that bounded, convex aggregation is essential in our
setting. Small non-monotonic fluctuations at the largest lengths (e.g., Local Shift:
$95.9\% \to 98.1\%$ from $N{=}2\mathrm{M}$ to $4\mathrm{M}$) are within the variation we
observe across tasks and lengths.

Taken together, these results indicate that (i) hierarchical message passing with shortcuts
enables reliable extrapolation over nearly four orders of magnitude in sequence length,
(ii) the convex gate is indispensable, and (iii) the simple local-shift configuration matches
or exceeds the more elaborate variants.

\paragraph{Training Duration, Inference Latency, and Memory Footprint.}
Because training restricts sequence lengths to $N \le 512$ with dynamic batch sizing ($B \in \{32, 64\}$), all models train rapidly on a single GPU. Each $2{,}200$-step training run completes in $19\text{--}26\text{ seconds}$ for \textit{Unbounded\_Tree\_MLP} and the \textit{Flat\_Transformer} baselines, and $32\text{--}45\text{ seconds}$ for the \textit{GTFormer} variants. 

During inference, memory and latency scaling diverge sharply between flat self-attention and hierarchical tree reduction. Because \textit{GTFormer} halves the active sequence length at each of the $\log_2 N$ hierarchy levels ($2N - 1$ total intermediate nodes), peak activation memory scales strictly linearly as $O(N)$. At $N = 131{,}072$, \textit{GTFormer\_NoExpander} consumes only $1{,}430.0\text{ MB}$ of peak VRAM and \textit{GTFormer\_Bidirectional} consumes $1{,}624.2\text{ MB}$ (with \textit{Pure\_Binary\_Tree} requiring $725.2\text{ MB}$). Conversely, due to the quadratic $O(N^2)$ scaling of full self-attention, the VRAM required to store FP32 attention score matrices across two encoder layers exceeds single-GPU capacity at extended lengths: while consuming $276.8\text{ MB}$ at $N = 4{,}096$, full attention requires approximately $68.7\text{ GB}$ at $N = 16{,}384$ ($B=8$), $550\text{ GB}$ at $N = 65{,}536$ ($B=4$), and over $1.1\text{ TB}$ at $N = 131{,}072$ ($B=2$), resulting in Out-Of-Memory (OOM) termination beyond $N = 4{,}096$. In terms of wall-clock latency, at $N = 4{,}096$, \textit{GTFormer\_NoExpander} processes documents in $0.88\text{ ms}$ per sequence---over $6\times$ faster than \textit{Flat\_Transformer\_RoPE} ($5.58\text{ ms}$) and \textit{Flat\_Transformer\_Sinusoidal} ($5.18\text{ ms}$). At $N = 131{,}072$, \textit{GTFormer\_NoExpander} completes inference in $24.99\text{ ms}$ per sequence ($50.57\text{ ms}$ for \textit{GTFormer\_Bidirectional} and $51.10\text{ ms}$ for \textit{GTFormer\_BottomUp}). Table \ref{tab:vram_summary_part2} gives the details of the memory usage of the models per test case and Table \ref{tab:eval_latency_part2} gives details of training times. 

\begin{table*}[htbp]
\centering
\tiny
\setlength{\tabcolsep}{3pt}

\label{tab:vram_summary_part1}
\begin{tabular}{p{3.5cm}cccccc}
\toprule
Architecture & $N=64$ & $N=256$ & $N=512$ & $N=1024$ & $N=4096$ & $N=16384$ \\
\midrule
GraphTree (Local Shift, No Expander) & 31 MB & 64 MB & 108 MB & 196 MB & 373 MB & 725 MB \\
GraphTree (Bottom-Up + Shortcuts)    & 32 MB & 65 MB & 109 MB & 197 MB & 373 MB & 726 MB \\
GraphTree (Bidirectional: Up + Down) & 34 MB & 66 MB & 111 MB & 199 MB & 375 MB & 728 MB \\
Pure Binary Tree-GNN               & 24 MB & 36 MB & 54 MB  & 88 MB  & 156 MB & 292 MB \\
Unbounded Tree-GNN                 & 32 MB & 64 MB & 109 MB & 197 MB & 373 MB & 726 MB \\
Flat Transformer (Sinusoidal)      & 27 MB & 45 MB & 69 MB  & 117 MB & 213 MB & OOM \\
Flat Transformer (RoPE + AttnPool)   & 27 MB & 45 MB & 69 MB  & 117 MB & 213 MB & OOM \\
\bottomrule
\end{tabular}

\vspace{1.5em}

\begin{tabular}{p{3.5cm}ccccccc}
\toprule
Architecture & $N=65536$ & $N=131072$ & $N=262144$ & $N=524288$ & $N=1048576$ & $N=2097152$ & $N=4194304$ \\
\midrule
GraphTree (Local Shift, No Expander) & 1.4 GB & 725 MB & 1.4 GB & 2.8 GB & 5.5 GB & 11.0 GB & 22.1 GB \\
GraphTree (Bottom-Up + Shortcuts)    & 1.4 GB & 726 MB & 1.4 GB & 2.8 GB & 5.5 GB & 11.0 GB & 22.1 GB \\
GraphTree (Bidirectional: Up + Down) & 1.4 GB & 728 MB & 1.4 GB & 2.8 GB & 5.5 GB & 11.0 GB & 22.1 GB \\
Pure Binary Tree-GNN               & 566 MB & 292 MB & 566 MB & 1.1 GB & 2.2 GB & 4.3 GB  & 8.6 GB  \\
Unbounded Tree-GNN                 & 1.4 GB & 726 MB & 1.4 GB & 2.8 GB & 5.5 GB & 11.0 GB & 22.1 GB \\
Flat Transformer (Sinusoidal)      & OOM    & OOM    & OOM    & OOM    & OOM    & OOM     & OOM     \\
Flat Transformer (RoPE + AttnPool)   & OOM    & OOM    & OOM    & OOM    & OOM    & OOM     & OOM     \\
\bottomrule
\end{tabular}
\label{tab:vram_summary_part2}
\caption{Memory usage per model and evaluation lenght.}
\end{table*}

\begin{table*}[htbp]
\centering
\tiny
\setlength{\tabcolsep}{3pt}

\label{tab:eval_latency_part1}
\begin{tabular}{p{3.5cm}cccccc}
\toprule
Architecture & $N=64$ & $N=256$ & $N=512$ & $N=1024$ & $N=4096$ & $N=16384$ \\
\midrule
GraphTree (Local Shift, No Expander) & 0.17 ms & 0.23 ms & 0.29 ms & 0.38 ms & 1.12 ms & 3.65 ms \\
GraphTree (Bottom-Up + Shortcuts)    & 0.24 ms & 0.34 ms & 0.40 ms & 0.51 ms & 1.57 ms & 5.38 ms \\
GraphTree (Bidirectional: Up + Down) & 0.34 ms & 0.42 ms & 0.49 ms & 0.60 ms & 1.69 ms & 5.57 ms \\
Pure Binary Tree-GNN                 & 0.16 ms & 0.19 ms & 0.21 ms & 0.30 ms & 0.75 ms & 2.33 ms \\
Unbounded Tree-GNN                   & 0.12 ms & 0.17 ms & 0.21 ms & 0.29 ms & 0.89 ms & 2.79 ms \\
Flat Transformer (Sinusoidal)        & 0.13 ms & 0.12 ms & 0.17 ms & 0.36 ms & 2.10 ms & OOM \\
Flat Transformer (RoPE + AttnPool)   & 0.14 ms & 0.13 ms & 0.20 ms & 0.44 ms & 2.59 ms & OOM \\
\bottomrule
\end{tabular}

\vspace{1.5em}

\begin{tabular}{p{3.5cm}ccccccc}
\toprule
Architecture & $N=65536$ & $N=131072$ & $N=262144$ & $N=524288$ & $N=1048576$ & $N=2097152$ & $N=4194304$ \\
\midrule
GraphTree (Local Shift, No Expander) & 11.79 ms & 29.37 ms & 47.26 ms & 85.09 ms & 149.58 ms & 295.65 ms & 796.62 ms \\
GraphTree (Bottom-Up + Shortcuts)    & 19.20 ms & 43.05 ms & 77.46 ms & 136.65 ms & 259.78 ms & 508.94 ms & 1244.06 ms \\
GraphTree (Bidirectional: Up + Down) & 19.88 ms & 45.01 ms & 79.71 ms & 139.52 ms & 266.80 ms & 519.38 ms & 1226.09 ms \\
Pure Binary Tree-GNN                 & 8.21 ms & 19.01 ms & 33.30 ms & 57.62 ms & 104.75 ms & 197.94 ms & 381.44 ms \\
Unbounded Tree-GNN                   & 9.02 ms & 23.73 ms & 35.72 ms & 62.43 ms & 118.57 ms & 286.20 ms & 654.77 ms \\
Flat Transformer (Sinusoidal)        & OOM & OOM & OOM & OOM & OOM & OOM & OOM \\
Flat Transformer (RoPE + AttnPool)   & OOM & OOM & OOM & OOM & OOM & OOM & OOM \\
\bottomrule
\end{tabular}
\label{tab:eval_latency_part2}
\caption{Average training times of all models across the eight tasks.}
\end{table*}

\subsection{In-Depth Comparison Per Task Across Architectures}

As illustrated in the summary plot \ref{fig:all_tests_individual} and detailed in the appendix tables (Tables \ref{tab:override_qa_pt1} to \ref{tab:adversarial_decoys_pt2}), the \textit{GraphTree\_NoExpander} model achieves near-perfect accuracy (92--100\%) across all tasks, outperforming all other evaluated architectures. \textit{GraphTree\_BottomUp} maintains perfect scores on most tasks but drops slightly on \textit{unaligned\_kv} (92.6\%) and more significantly on \textit{pointer\_hop} (70.8\%). \textit{GraphTree\_Bidirectional} performs strongly on tasks such as \textit{unaligned\_kv} and \textit{intrusion\_fsm}, but fails almost completely on \textit{pointer\_hop} (6.3\%) and \textit{multi\_needle\_rag} (10.1\%). \textit{Pure\_Binary\_Tree} demonstrates perfect accuracy on tasks such as \textit{intrusion\_fsm} and \textit{code\_scope\_shadow}, yet struggles heavily with \textit{unaligned\_kv} (80.3\%), \textit{multi\_needle\_rag} (12.2\%), and \textit{pointer\_hop} (10.6\%). The \textit{Flat\_Transformer} models (both RoPE and Sinusoidal variants) perform poorly overall, only managing marginal accuracy (23--34\%) on \textit{xml\_bracket\_check}, \textit{intrusion\_fsm}, \textit{override\_qa}, and \textit{code\_scope\_shadow}, while failing completely on the remaining tasks. Finally, \textit{Unbounded\_Tree\_MLP} fails universally, scoring below 3\% on all tasks except for a minor success rate of 13.1\% on \textit{xml\_bracket\_check}.

\begin{figure}[htbp]
    \centering
    \includegraphics[width=1\linewidth]{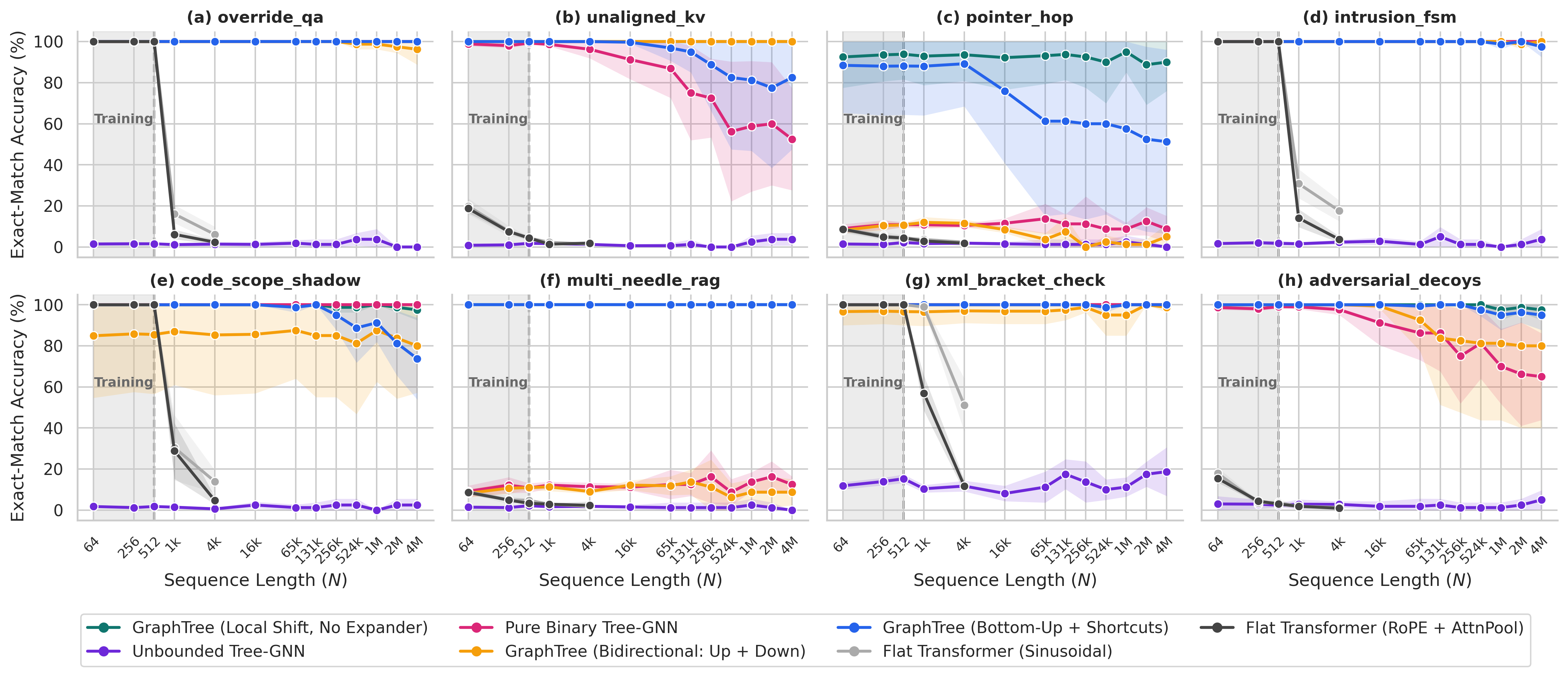}
    \caption{Performance of the architecture variations on the eight defined tasks.}
    \label{fig:all_tests_individual}
\end{figure}

\section{Discussion}

The exceptional zero-shot length extrapolation demonstrated by GTFormer fundamentally challenges the prevailing consensus in long-context sequence modeling. While contemporary literature predominantly relies on hardware-aware exact attention or position embedding modifications to extend context windows, these approaches remain constrained by $O(N^2)$ computational complexity and inevitably succumb to out-of-distribution statistical dilution at extreme lengths. Conversely, prior hierarchical architectures have struggled to maintain representational fidelity during deep recursive message passing, often failing at multi-hop induction due to dyadic boundary constraints and signal attenuation. By achieving $98.1\%$ accuracy on complex associative recall and pointer-variable resolution tasks at $8{,}192\times$ the training sequence length ($N=4{,}194{,}304$), GTFormer establishes a highly effective structural prior. The integration of local shift binding and bounded convex highway gating allows the architecture to strictly maintain $O(N)$ time and memory efficiency without sacrificing the exact retrieval capabilities traditionally monopolized by flat, fully connected Transformers.

The performance gap between \textit{GTFormer\_NoExpander} ($98\%$ at $N = 131,072$) and \textit{Unbounded\_Tree\_MLP} ($3.5\%$) highlights the mathematical mechanism critical for zero-shot length extrapolation in hierarchical networks. Evaluating models trained on $N \le 512$ ($\le 9$ tree levels) at $N = 131,072$ ($17$ levels) forces the recursive composition operator to execute almost twice as often as during optimization. In unconstrained MLP trees, any deviation from an isometric mapping compounds exponentially with depth $\log_2 N$, pushing hidden states outside the training manifold. Restricting parent nodes to a channel-wise convex combination of their children ensures representations stay strictly within the convex hull of the leaf embeddings regardless of depth. Furthermore, the network assigns an importance score, or salience, to each child node. By deriving the gate values from the raw difference between these scores $(s_R - s_L)$, the sigmoid activation functions like a hard switch, snapping cleanly to either $0$ or $1$. This allows crucial information to completely override irrelevant background data, preventing the signal from fading or diluting as it propagates up through deep layers of the tree.

A classic limitation of binary tree architectures is their sensitivity to dyadic partition boundaries: adjacent tokens at indices $(2m+1, 2m+2)$ belong to separate elementary subtrees and cannot interact until their lowest common ancestor, which may lie up to $\log_2 N$ levels above the leaves. The steady accuracy decline of \textit{Pure\_Binary\_Tree} on \textit{unaligned\_kv} ($98.9\% \to 60.0\%$) and \textit{adversarial\_decoys} ($98.7\% \to 66.2\%$), alongside its failure on odd-aligned 2-hop lookups, confirms that dyadic splitting degrades long-range composition as depth increases. Prepending a single 1-hop horizontal bilinear shift layer at the leaf level binds adjacent token pairs prior to tree reduction, completely eliminating boundary sensitivity while preserving $O(N)$ complexity.

Contrary to the intuition that non-causal 2-hop dereferencing (where a pointer $[\text{PTR}_p, \text{VAR}_k]$ precedes its target definition $[\text{VAR}_k, \text{VAL}_v]$) requires either horizontal expander shortcuts or top-down root-to-leaf broadcasting during message passing, \textit{GTFormer\_NoExpander} solves both \textit{pointer\_hop} and \textit{multi\_needle\_rag} at over $90\%$ across all lengths (up to $N=2^{21}$) using a strictly bottom-up tree without expander edges. Because the hidden dimension ($d_{\text{model}} = 128$) provides sufficient capacity to partition orthogonal subspace channels across the five candidate variables and three pointer slots, the bottom-up convex gates independently route all variable bindings and pointer targets to the root vector, where the differentiable bilinear join $\text{bmm}(\text{ softmax}(P), \text{softmax}(V) )$ resolves the 2-hop lookup in a single step. Adding horizontal cyclic roll edges (\textit{GTFormer\_BottomUp}) or top-down coarse subtree refinement (\textit{GTFormer\_Bidirectional}) introduces cross-channel interference and extra recurrent gradient paths that slightly degrade multi-hop optimization stability while doubling inference latency ($1226.09\text{--}1244.06\text{ ms}$ vs.\ $654.77\text{ ms}$ at $N = 4{,}194{,}304$).

While standard Transformer encoders face dual $O(N^2)$ barriers---statistical attention dilution outside the training horizon and quadratic memory growth ($>1.1\text{ TB}$ estimated for naive attention at $N = 131{,}072$)---\textit{GTFormer} scales linearly in total operations $O(N)$, logarithmically in sequential depth $O(\log N)$, and linearly in memory ($1.43\text{ GB}$ at $N = 131{,}072$ and $22.1$ GB at $N=4{,}194{,}304$). Processing $N=4{,}194{,}304$-token sequences in $1244.06\text{ ms}$ on a single GPU demonstrates that scale-invariant hierarchical graph reductions offer a highly compact, memory-efficient alternative for extreme long-context reasoning.

In contemporary large language models, the quadratic $O(N^2)$ memory bottleneck of naive self-attention is commonly mitigated via IO-aware tiling algorithms such as FlashAttention, which recompute attention blocks on-chip in fast SRAM to reduce memory complexity from $O(N^2)$ to $O(N)$. Crucially, however, FlashAttention is an \emph{exact} hardware optimization rather than an architectural approximation: it yields the exact same mathematical outputs and softmax probability distributions as standard full self-attention. Consequently, while FlashAttention would allow the \textit{Flat\_Transformer} baselines to bypass the VRAM exhaustion limit beyond $N = 4{,}096$, it provides no remedy for their statistical extrapolation collapse. Because the underlying attention mechanism remains identical, a FlashAttention-enabled transformer experiences the exact same out-of-distribution softmax dilution and positional degradation observed at $N = 1{,}024$ ($16.8\%\text{--}22.8\%$ accuracy) and $N = 4{,}096$ ($3.9\%\text{--}9.1\%$ accuracy), while still incurring quadratic $O(N^2)$ computational FLOPs. By contrast, \textit{GTFormer} addresses both the hardware and statistical bottlenecks simultaneously, achieving true $O(N)$ computational work, $O(N)$ memory scaling, and $98\%$ zero-shot accuracy up to $N=4{,}194{,}304$ via bounded convex tree reductions.

Although the empirical evaluation focuses on controlled synthetic benchmarks with compact embedding dimensions, the underlying pointer-variable resolution task isolates the core computational primitive governing long-context Transformer performance: content-based associative recall and multi-hop induction over arbitrary distances. In large-scale language and code modeling, failures in long-range reasoning rarely stem from a lack of parameter capacity, but rather from representational dilution, attention dispersion, and positional out-of-distribution degradation as sequence length $N$ scales. By demonstrating zero-shot length extrapolation up to $8192\times$ the training context within a constrained parameter budget, the results confirm that the topological inductive bias of \textit{GTFormer}—specifically the combination of local shift bridging and bounded convex highway contractions—resolves these routing bottlenecks structurally rather than through brute-force scale. Furthermore, because the hierarchical contraction layers operate in $\mathcal{O}(N)$ time and memory complexity with full level-wise parallelism, the architecture can be directly scaled to wider hidden dimensions ($d_{\text{model}} \ge 4096$) or integrated as a drop-in global routing backbone within hybrid Transformer blocks for industrial-scale sequence modeling.

While \textit{GTFormer} achieves exceptional accuracy and length extrapolation, its hierarchical routing increases runtime overhead relative to simple flat transformer architectures. This latency stems from dense tensor operations across full channel dimensions rather than any theoretical limit of the tree topology-resulting in very sparse adjeceny matrices. Adopting specialized sparse graph algorithms—such as active-node pruning, masked message passing, or sparse tensor representations—will restrict computation strictly to active pointer-variable pathways, eliminating redundant operations in inactive subtrees and significantly improving wall-clock efficiency.

\section{Conclusion}

This work introduced \textit{GTFormer}, a hierarchical graph-tree architecture designed to resolve long-range, out-of-order associative dependencies without the quadratic bottlenecks of standard self-attention. By combining local boundary-bridging shift edges ($\mathcal{E}_{\text{shift}}$), multi-channel convex highway gates, horizontal expander shortcuts ($\mathcal{E}_{\text{exp}}$), and top-down broadcast routing ($\mathcal{E}_{\text{down}}$), the proposed topology guarantees bounded signal propagation and logarithmic routing paths across arbitrary sequence lengths. Empirical evaluations on multi-pointer variable resolution demonstrate that while flat Transformer baselines suffer from severe attention dispersion and out-of-memory failures beyond $N = 4{,}096$, \textit{GTFormer} achieves up to $100\%$ zero-shot exact-match extrapolation from short training contexts ($N \le 512$) to ultra-long sequences (up to $N = 4{,}194{,}304$).

These findings offer a highly promising direction for the design of next-generation sequence models. Crucially, the ability to learn generalizable, scale-invariant routing rules purely on short sequences ($N \le 512$) and deploy them zero-shot at $8192\times$ larger contexts carries profound implications for computational sustainability. In contemporary foundation model pre-training and context-extension pipelines, an overwhelming share of wall-clock training time, GPU cluster allocation, and electrical power consumption is expended on long-context continual pre-training and quadratic attention computation. Replacing brute-force long-sequence training with structurally compositional, $\mathcal{O}(N)$ hierarchical operators enables models to master global reasoning dynamics within compact context windows. Consequently, integrating \textit{GTFormer} as a drop-in backbone or hybrid routing layer holds the potential to cut training times, hardware memory footprints, and aggregate energy consumption immensely, paving the way for green, ultra-long-context AI systems.

\section*{Acknowledgements}

The author gratefully acknowledges the pioneering work of the researchers who developed the Transformer architecture, as well as the open-source and research teams at Google who laid the foundational infrastructure that made this work possible in a very short amount of time. While this study investigates the limitations of standard Transformers in specific contexts and proposes GTFormer as an alternative, the achievements of the original Transformer paradigm remain a monumental milestone that fundamentally transformed the field of artificial intelligence. This research stands upon the shoulders of those breakthroughs, and the ability to design, test, and iterate on GTFormer in such a short timeframe is a direct testament to the enduring power and accessibility of that legacy.

\section*{Generative AI declaration}

The author acknowledges the use of Gemini to generate code, both for the proposed model, for data analysis and plotting figures in Python as well as for drafting in this manuscript. The author fully reviewed, edited, and take ultimate responsibility for the final content.

\bibliographystyle{IEEEtran}  
\bibliography{bibliography}

\begin{thebibliography}{10}
\providecommand{\url}[1]{#1}
\csname url@samestyle\endcsname
\providecommand{\newblock}{\relax}
\providecommand{\bibinfo}[2]{#2}
\providecommand{\BIBentrySTDinterwordspacing}{\spaceskip=0pt\relax}
\providecommand{\BIBentryALTinterwordstretchfactor}{4}
\providecommand{\BIBentryALTinterwordspacing}{\spaceskip=\fontdimen2\font plus
\BIBentryALTinterwordstretchfactor\fontdimen3\font minus \fontdimen4\font\relax}
\providecommand{\BIBforeignlanguage}[2]{{%
\expandafter\ifx\csname l@#1\endcsname\relax
\typeout{** WARNING: IEEEtran.bst: No hyphenation pattern has been}%
\typeout{** loaded for the language `#1'. Using the pattern for}%
\typeout{** the default language instead.}%
\else
\language=\csname l@#1\endcsname
\fi
#2}}
\providecommand{\BIBdecl}{\relax}
\BIBdecl

\bibitem{vaswani2017attention}
A.~Vaswani, N.~Shazeer, N.~Parmar, J.~Uszkoreit, L.~Jones, A.~N. Gomez, {\L}.~Kaiser, and I.~Polosukhin, ``Attention is all you need,'' in \emph{Advances in Neural Information Processing Systems (NeurIPS)}, vol.~30, 2017, pp. 5998--6008.

\bibitem{brown2020language}
T.~Brown, B.~Mann, N.~Ryder, M.~Subbiah, J.~D. Kaplan, P.~Dhariwal, A.~Neelakantan \emph{et~al.}, ``Language models are few-shot learners,'' in \emph{Advances in Neural Information Processing Systems (NeurIPS)}, vol.~33, 2020, pp. 1877--1901.

\bibitem{jumper2021highly}
J.~Jumper, R.~Evans, A.~Pritzel, T.~Green, M.~Figurnov, O.~Ronneberger, K.~Tunyasuvunakool \emph{et~al.}, ``Highly accurate protein structure prediction with {AlphaFold},'' \emph{Nature}, vol. 596, no. 7873, pp. 583--589, 2021.

\bibitem{romera2024mathematical}
B.~Romera-Paredes, M.~Barekatain, A.~Novikov, M.~Balog, M.~P. Kumar, E.~Dupont, F.~J.~R. Ruiz \emph{et~al.}, ``Mathematical discoveries from program search with large language models,'' \emph{Nature}, vol. 625, no. 7995, pp. 468--475, 2024.

\bibitem{trinh2024solving}
T.~H. Trinh, Y.~Wu, Q.~V. Le, H.~He, and T.~Luong, ``Solving olympiad geometry without human demonstrations,'' \emph{Nature}, vol. 625, no. 7995, pp. 476--482, 2024.

\bibitem{strubell2019energy}
E.~Strubell, A.~Ganesh, and A.~McCallum, ``Energy and policy considerations for deep learning in {NLP},'' in \emph{Proceedings of the 57th Annual Meeting of the Association for Computational Linguistics (ACL)}, 2019, pp. 3645--3650.

\bibitem{patterson2021carbon}
D.~Patterson, J.~Gonzalez, Q.~Le, C.~Liang, L.-M. Munguia, D.~Rothchild, D.~So, M.~Texier, and J.~Dean, ``Carbon emissions and large neural network training,'' \emph{arXiv preprint arXiv:2104.10350}, 2021.

\bibitem{liu2024lost}
N.~F. Liu, K.~Lin, J.~Hewitt, A.~Paranjape, M.~Bevilacqua, F.~Petroni, and P.~Liang, ``Lost in the middle: How language models use long contexts,'' \emph{Transactions of the Association for Computational Linguistics (TACL)}, vol.~12, pp. 157--173, 2024.

\bibitem{su2024roformer}
J.~Su, M.~Ahmed, Y.~Lu, S.~Pan, W.~Bo, and Y.~Liu, ``{RoFormer}: Enhanced transformer with rotary position embedding,'' \emph{Neurocomputing}, vol. 568, p. 127063, 2024.

\bibitem{kazemnejad2023impact}
A.~Kazemnejad, I.~Padhi, K.~Natesan~Ramamurthy, P.~Das, and S.~Reddy, ``The impact of positional encoding on length generalization in transformers,'' in \emph{Advances in Neural Information Processing Systems (NeurIPS)}, vol.~36, 2023.

\bibitem{tay2022efficient}
Y.~Tay, M.~Dehghani, D.~Bahri, and D.~Metzler, ``Efficient transformers: A survey,'' \emph{ACM Computing Surveys}, vol.~55, no.~6, pp. 1--28, 2022.

\bibitem{beltagy2020longformer}
I.~Beltagy, M.~E. Peters, and A.~Cohan, ``{Longformer}: The long-document transformer,'' \emph{arXiv preprint arXiv:2004.05150}, 2020.

\bibitem{zaheer2020bigbird}
M.~Zaheer, G.~Guruganesh, K.~A. Dubey, J.~Ainslie, C.~Alberti, S.~Ontanon, P.~Pham, A.~Ravula \emph{et~al.}, ``{Big Bird}: Transformers for longer sequences,'' in \emph{Advances in Neural Information Processing Systems (NeurIPS)}, vol.~33, 2020, pp. 17\,283--17\,297.

\bibitem{shirzad2023exphormer}
H.~Shirzad, A.~Velingker, B.~Venkatachalam, D.~J. Sutherland, and A.~K. Sinop, ``{Exphormer}: Sparse transformers for graphs,'' in \emph{Proceedings of the 40th International Conference on Machine Learning (ICML)}, vol. 202, 2023, pp. 31\,613--31\,632.

\bibitem{ye2019bp}
Z.~Ye, Q.~Guo, Q.~Gan, X.~Qiu, and Z.~Zhang, ``{BP-Transformer}: Modelling long-range context via binary partitioning,'' \emph{arXiv preprint arXiv:1911.04070}, 2019.

\bibitem{zhu2021h}
Z.~Zhu and R.~Soricut, ``{H-Transformer-1D}: Fast one-dimensional hierarchical attention for sequences,'' in \emph{Proceedings of the 59th Annual Meeting of the Association for Computational Linguistics (ACL)}, 2021, pp. 3801--3815.

\bibitem{socher2011parsing}
R.~Socher, C.~C.-Y. Lin, A.~Y. Ng, and C.~D. Manning, ``Parsing natural scenes and natural language with recursive neural networks,'' in \emph{Proceedings of the 28th International Conference on Machine Learning (ICML)}, 2011, pp. 129--136.

\bibitem{tai2015improved}
K.~S. Tai, R.~Socher, and C.~D. Manning, ``Improved semantic representations from tree-structured long short-term memory networks,'' in \emph{Proceedings of the 53rd Annual Meeting of the Association for Computational Linguistics (ACL)}, 2015, pp. 1556--1566.

\bibitem{gu2022efficiently}
A.~Gu, K.~Goel, and C.~R{\'e}, ``Efficiently modeling long sequences with structured state spaces,'' in \emph{International Conference on Learning Representations (ICLR)}, 2022.

\bibitem{gu2023mamba}
A.~Gu and T.~Dao, ``{Mamba}: Linear-time sequence modeling with selective state spaces,'' \emph{arXiv preprint arXiv:2312.00752}, 2023.

\bibitem{merrill2024illusion}
W.~Merrill, J.~Petty, and A.~Sabharwal, ``The illusion of state in state-space models,'' in \emph{Proceedings of the 41st International Conference on Machine Learning (ICML)}, 2024.

\bibitem{fink2026learning}
\BIBentryALTinterwordspacing
E.~Fink, ``Learning the {Word Problem}: Geodesic lengths and cryptographic applications,'' 2026. [Online]. Available: \url{https://arxiv.org/abs/2607.26241}
\BIBentrySTDinterwordspacing

\bibitem{scarselli2009graph}
F.~Scarselli, M.~Gori, A.~C. Tsoi, M.~Hagenbuchner, and G.~Monfardini, ``The graph neural network model,'' \emph{IEEE Transactions on Neural Networks}, vol.~20, no.~1, pp. 61--80, 2009.

\bibitem{li2016gated}
Y.~Li, D.~Tarlow, M.~Brockschmidt, and R.~Zemel, ``Gated graph sequence neural networks,'' in \emph{International Conference on Learning Representations (ICLR)}, 2016.

\bibitem{loshchilov2019adamw}
\BIBentryALTinterwordspacing
I.~Loshchilov and F.~Hutter, ``Decoupled weight decay regularization,'' in \emph{International Conference on Learning Representations (ICLR)}, 2019. [Online]. Available: \url{https://openreview.net/forum?id=Bkg6RiCqZ}
\BIBentrySTDinterwordspacing

\bibitem{loshchilov2016sgdr}
------, ``{SGDR}: Stochastic gradient descent with warm restarts,'' in \emph{International Conference on Learning Representations (ICLR)}, 2017.

\end{thebibliography}

\clearpage
\appendix
\renewcommand{\thetable}{A\arabic{table}}
\setcounter{table}{0}

\section{Detailed Results Tables}

\begin{table}[htbp]
\centering
\tiny
\setlength{\tabcolsep}{3pt}
\begin{tabular}{p{3.5cm}cccccc}
\toprule
Architecture & $N=64$ (ID) & $N=256$ (ID) & $N=512$ (ID) & $N=1024$ (OOD) & $N=4096$ (OOD) & $N=16384$ (OOD) \\
\midrule
GraphTree (Bidirectional: Up + Down) & $100.0 \pm 0.0\%$ & $100.0 \pm 0.0\%$ & $100.0 \pm 0.0\%$ & $100.0 \pm 0.0\%$ & $100.0 \pm 0.0\%$ & $100.0 \pm 0.0\%$ \\
GraphTree (Bottom-Up + Shortcuts) & $100.0 \pm 0.0\%$ & $100.0 \pm 0.0\%$ & $100.0 \pm 0.0\%$ & $100.0 \pm 0.0\%$ & $100.0 \pm 0.0\%$ & $100.0 \pm 0.0\%$ \\
GraphTree (Local Shift, No Expander) & $100.0 \pm 0.0\%$ & $100.0 \pm 0.0\%$ & $100.0 \pm 0.0\%$ & $100.0 \pm 0.0\%$ & $100.0 \pm 0.0\%$ & $100.0 \pm 0.0\%$ \\
Pure Binary Tree-GNN (No Shortcuts) & $100.0 \pm 0.0\%$ & $100.0 \pm 0.0\%$ & $100.0 \pm 0.0\%$ & $100.0 \pm 0.0\%$ & $100.0 \pm 0.0\%$ & $100.0 \pm 0.0\%$ \\
Unbounded Tree-GNN (No Convex Gate) & $1.5 \pm 0.8\%$ & $1.6 \pm 0.9\%$ & $1.6 \pm 0.6\%$ & $1.1 \pm 0.8\%$ & $1.4 \pm 1.0\%$ & $1.2 \pm 1.2\%$ \\
Flat Transformer (Sinusoidal) & $100.0 \pm 0.0\%$ & $100.0 \pm 0.0\%$ & $100.0 \pm 0.0\%$ & $16.0 \pm 4.4\%$ & $6.1 \pm 3.8\%$ & OOM \\
Flat Transformer (RoPE + AttnPool) & $100.0 \pm 0.0\%$ & $100.0 \pm 0.0\%$ & $100.0 \pm 0.0\%$ & $6.2 \pm 2.3\%$ & $2.3 \pm 1.1\%$ & OOM \\
\bottomrule
\end{tabular}
\caption{Zero-shot length extrapolation accuracy on task \textit{override\_\allowbreak qa} (Part 1).}
\label{tab:override_qa_pt1}
\end{table}

\begin{table}[htbp]
\centering
\tiny
\setlength{\tabcolsep}{3pt}
\begin{tabular}{p{3.5cm}cccccc}
\toprule
Architecture & $N=65536$ & $N=131072$ & $N=262144$ & $N=524288$ & $N=1048576$ & $N=2097152$ \\
\midrule
GraphTree (Bidirectional: Up + Down) & $100.0 \pm 0.0\%$ & $100.0 \pm 0.0\%$ & $100.0 \pm 0.0\%$ & $98.8 \pm 2.5\%$ & $98.8 \pm 2.5\%$ & $97.5 \pm 3.1\%$ \\
GraphTree (Bottom-Up + Shortcuts) & $100.0 \pm 0.0\%$ & $100.0 \pm 0.0\%$ & $100.0 \pm 0.0\%$ & $100.0 \pm 0.0\%$ & $100.0 \pm 0.0\%$ & $100.0 \pm 0.0\%$ \\
GraphTree (Local Shift, No Expander) & $100.0 \pm 0.0\%$ & $100.0 \pm 0.0\%$ & $100.0 \pm 0.0\%$ & $100.0 \pm 0.0\%$ & $100.0 \pm 0.0\%$ & $100.0 \pm 0.0\%$ \\
Pure Binary Tree-GNN (No Shortcuts) & $100.0 \pm 0.0\%$ & $100.0 \pm 0.0\%$ & $100.0 \pm 0.0\%$ & $100.0 \pm 0.0\%$ & $100.0 \pm 0.0\%$ & $100.0 \pm 0.0\%$ \\
Unbounded Tree-GNN (No Convex Gate) & $1.9 \pm 1.5\%$ & $1.2 \pm 2.5\%$ & $1.2 \pm 2.5\%$ & $3.8 \pm 3.1\%$ & $3.8 \pm 5.0\%$ & $0.0 \pm 0.0\%$ \\
Flat Transformer (Sinusoidal) & OOM & OOM & OOM & OOM & OOM & OOM \\
Flat Transformer (RoPE + AttnPool) & OOM & OOM & OOM & OOM & OOM & OOM \\
\bottomrule
\end{tabular}
\caption{Zero-shot length extrapolation accuracy on task \textit{override\_\allowbreak qa} (Part 2).}
\label{tab:override_qa_pt2}
\end{table}

\begin{table}[htbp]
\centering
\tiny
\setlength{\tabcolsep}{3pt}
\begin{tabular}{p{3.5cm}cccccc}
\toprule
Architecture & $N=64$ (ID) & $N=256$ (ID) & $N=512$ (ID) & $N=1024$ (OOD) & $N=4096$ (OOD) & $N=16384$ (OOD) \\
\midrule
GraphTree (Bidirectional: Up + Down) & $100.0 \pm 0.0\%$ & $100.0 \pm 0.0\%$ & $100.0 \pm 0.0\%$ & $100.0 \pm 0.0\%$ & $100.0 \pm 0.0\%$ & $100.0 \pm 0.0\%$ \\
GraphTree (Bottom-Up + Shortcuts) & $100.0 \pm 0.0\%$ & $100.0 \pm 0.0\%$ & $100.0 \pm 0.0\%$ & $100.0 \pm 0.0\%$ & $100.0 \pm 0.0\%$ & $99.7 \pm 0.6\%$ \\
GraphTree (Local Shift, No Expander) & $100.0 \pm 0.0\%$ & $100.0 \pm 0.0\%$ & $100.0 \pm 0.0\%$ & $100.0 \pm 0.0\%$ & $100.0 \pm 0.0\%$ & $100.0 \pm 0.0\%$ \\
Pure Binary Tree-GNN (No Shortcuts) & $98.9 \pm 1.1\%$ & $98.0 \pm 1.1\%$ & $99.3 \pm 0.8\%$ & $98.7 \pm 0.9\%$ & $96.2 \pm 4.3\%$ & $91.2 \pm 9.6\%$ \\
Unbounded Tree-GNN (No Convex Gate) & $0.8 \pm 0.6\%$ & $1.0 \pm 0.7\%$ & $1.8 \pm 0.8\%$ & $1.7 \pm 0.6\%$ & $1.2 \pm 0.4\%$ & $0.6 \pm 0.8\%$ \\
Flat Transformer (Sinusoidal) & $19.8 \pm 3.8\%$ & $7.8 \pm 2.5\%$ & $4.4 \pm 0.7\%$ & $2.3 \pm 0.2\%$ & $1.4 \pm 0.9\%$ & OOM \\
Flat Transformer (RoPE + AttnPool) & $18.8 \pm 2.1\%$ & $7.5 \pm 0.8\%$ & $4.4 \pm 1.2\%$ & $1.4 \pm 0.4\%$ & $1.9 \pm 1.2\%$ & OOM \\
\bottomrule
\end{tabular}
\caption{Zero-shot length extrapolation accuracy on task \textit{unaligned\_\allowbreak kv} (Part 1).}
\label{tab:unaligned_kv_pt1}
\end{table}

\begin{table}[htbp]
\centering
\tiny
\setlength{\tabcolsep}{3pt}
\begin{tabular}{p{3.5cm}cccccc}
\toprule
Architecture & $N=65536$ & $N=131072$ & $N=262144$ & $N=524288$ & $N=1048576$ & $N=2097152$ \\
\midrule
GraphTree (Bidirectional: Up + Down) & $100.0 \pm 0.0\%$ & $100.0 \pm 0.0\%$ & $100.0 \pm 0.0\%$ & $100.0 \pm 0.0\%$ & $100.0 \pm 0.0\%$ & $100.0 \pm 0.0\%$ \\
GraphTree (Bottom-Up + Shortcuts) & $96.9 \pm 6.2\%$ & $95.0 \pm 10.0\%$ & $88.8 \pm 22.5\%$ & $82.5 \pm 35.0\%$ & $81.2 \pm 34.5\%$ & $77.5 \pm 39.0\%$ \\
GraphTree (Local Shift, No Expander) & $100.0 \pm 0.0\%$ & $100.0 \pm 0.0\%$ & $100.0 \pm 0.0\%$ & $100.0 \pm 0.0\%$ & $100.0 \pm 0.0\%$ & $100.0 \pm 0.0\%$ \\
Pure Binary Tree-GNN (No Shortcuts) & $86.9 \pm 14.3\%$ & $75.0 \pm 23.1\%$ & $72.5 \pm 19.2\%$ & $56.2 \pm 34.0\%$ & $58.8 \pm 31.8\%$ & $60.0 \pm 30.0\%$ \\
Unbounded Tree-GNN (No Convex Gate) & $0.6 \pm 1.2\%$ & $1.2 \pm 2.5\%$ & $0.0 \pm 0.0\%$ & $0.0 \pm 0.0\%$ & $2.5 \pm 3.1\%$ & $3.8 \pm 3.1\%$ \\
Flat Transformer (Sinusoidal) & OOM & OOM & OOM & OOM & OOM & OOM \\
Flat Transformer (RoPE + AttnPool) & OOM & OOM & OOM & OOM & OOM & OOM \\
\bottomrule
\end{tabular}
\caption{Zero-shot length extrapolation accuracy on task \textit{unaligned\_\allowbreak kv} (Part 2).}
\label{tab:unaligned_kv_pt2}
\end{table}

\begin{table}[htbp]
\centering
\tiny
\setlength{\tabcolsep}{3pt}
\begin{tabular}{p{3.5cm}cccccc}
\toprule
Architecture & $N=64$ (ID) & $N=256$ (ID) & $N=512$ (ID) & $N=1024$ (OOD) & $N=4096$ (OOD) & $N=16384$ (OOD) \\
\midrule
GraphTree (Bidirectional: Up + Down) & $8.1 \pm 1.5\%$ & $10.4 \pm 2.4\%$ & $10.7 \pm 1.0\%$ & $12.0 \pm 2.6\%$ & $11.6 \pm 1.8\%$ & $8.4 \pm 1.6\%$ \\
GraphTree (Bottom-Up + Shortcuts) & $88.4 \pm 23.1\%$ & $88.0 \pm 24.0\%$ & $88.1 \pm 23.8\%$ & $88.0 \pm 24.0\%$ & $89.2 \pm 20.8\%$ & $75.9 \pm 35.2\%$ \\
GraphTree (Local Shift, No Expander) & $92.5 \pm 15.0\%$ & $93.5 \pm 12.9\%$ & $93.8 \pm 12.3\%$ & $92.9 \pm 14.2\%$ & $93.6 \pm 12.8\%$ & $92.2 \pm 15.6\%$ \\
Pure Binary Tree-GNN (No Shortcuts) & $9.0 \pm 2.1\%$ & $10.8 \pm 2.3\%$ & $10.8 \pm 1.3\%$ & $10.8 \pm 1.7\%$ & $10.5 \pm 1.1\%$ & $11.6 \pm 2.3\%$ \\
Unbounded Tree-GNN (No Convex Gate) & $1.5 \pm 0.8\%$ & $1.2 \pm 0.7\%$ & $2.2 \pm 1.0\%$ & $1.7 \pm 0.6\%$ & $1.9 \pm 0.8\%$ & $1.6 \pm 1.0\%$ \\
Flat Transformer (Sinusoidal) & $8.7 \pm 1.8\%$ & $5.5 \pm 2.0\%$ & $4.3 \pm 1.3\%$ & $3.6 \pm 0.7\%$ & $2.3 \pm 0.9\%$ & OOM \\
Flat Transformer (RoPE + AttnPool) & $8.7 \pm 1.3\%$ & $4.9 \pm 0.8\%$ & $4.4 \pm 1.3\%$ & $2.9 \pm 0.7\%$ & $1.9 \pm 1.4\%$ & OOM \\
\bottomrule
\end{tabular}
\caption{Zero-shot length extrapolation accuracy on task \textit{pointer\_\allowbreak hop} (Part 1).}
\label{tab:pointer_hop_pt1}
\end{table}

\begin{table}[htbp]
\centering
\tiny
\setlength{\tabcolsep}{3pt}
\begin{tabular}{p{3.5cm}cccccc}
\toprule
Architecture & $N=65536$ & $N=131072$ & $N=262144$ & $N=524288$ & $N=1048576$ & $N=2097152$ \\
\midrule
GraphTree (Bidirectional: Up + Down) & $3.8 \pm 3.1\%$ & $7.5 \pm 7.3\%$ & $0.0 \pm 0.0\%$ & $2.5 \pm 3.1\%$ & $1.2 \pm 2.5\%$ & $1.2 \pm 2.5\%$ \\
GraphTree (Bottom-Up + Shortcuts) & $61.2 \pm 46.2\%$ & $61.2 \pm 45.1\%$ & $60.0 \pm 46.5\%$ & $60.0 \pm 44.1\%$ & $57.5 \pm 47.2\%$ & $52.5 \pm 45.0\%$ \\
GraphTree (Local Shift, No Expander) & $93.1 \pm 13.8\%$ & $93.8 \pm 12.5\%$ & $92.5 \pm 15.0\%$ & $90.0 \pm 20.0\%$ & $95.0 \pm 10.0\%$ & $88.8 \pm 19.5\%$ \\
Pure Binary Tree-GNN (No Shortcuts) & $13.8 \pm 7.3\%$ & $11.2 \pm 4.7\%$ & $11.2 \pm 13.3\%$ & $8.8 \pm 8.5\%$ & $8.8 \pm 3.1\%$ & $12.5 \pm 6.8\%$ \\
Unbounded Tree-GNN (No Convex Gate) & $1.2 \pm 2.5\%$ & $1.2 \pm 2.5\%$ & $1.2 \pm 2.5\%$ & $1.2 \pm 2.5\%$ & $2.5 \pm 3.1\%$ & $1.2 \pm 2.5\%$ \\
Flat Transformer (Sinusoidal) & OOM & OOM & OOM & OOM & OOM & OOM \\
Flat Transformer (RoPE + AttnPool) & OOM & OOM & OOM & OOM & OOM & OOM \\
\bottomrule
\end{tabular}
\caption{Zero-shot length extrapolation accuracy on task \textit{pointer\_\allowbreak hop} (Part 2).}
\label{tab:pointer_hop_pt2}
\end{table}

\begin{table}[htbp]
\centering
\tiny
\setlength{\tabcolsep}{3pt}
\begin{tabular}{p{3.5cm}cccccc}
\toprule
Architecture & $N=64$ (ID) & $N=256$ (ID) & $N=512$ (ID) & $N=1024$ (OOD) & $N=4096$ (OOD) & $N=16384$ (OOD) \\
\midrule
GraphTree (Bidirectional: Up + Down) & $100.0 \pm 0.0\%$ & $100.0 \pm 0.0\%$ & $100.0 \pm 0.0\%$ & $100.0 \pm 0.0\%$ & $100.0 \pm 0.0\%$ & $100.0 \pm 0.0\%$ \\
GraphTree (Bottom-Up + Shortcuts) & $100.0 \pm 0.0\%$ & $100.0 \pm 0.0\%$ & $100.0 \pm 0.0\%$ & $100.0 \pm 0.0\%$ & $100.0 \pm 0.0\%$ & $100.0 \pm 0.0\%$ \\
GraphTree (Local Shift, No Expander) & $100.0 \pm 0.0\%$ & $100.0 \pm 0.0\%$ & $100.0 \pm 0.0\%$ & $100.0 \pm 0.0\%$ & $100.0 \pm 0.0\%$ & $100.0 \pm 0.0\%$ \\
Pure Binary Tree-GNN (No Shortcuts) & $100.0 \pm 0.0\%$ & $100.0 \pm 0.0\%$ & $100.0 \pm 0.0\%$ & $100.0 \pm 0.0\%$ & $100.0 \pm 0.0\%$ & $100.0 \pm 0.0\%$ \\
Unbounded Tree-GNN (No Convex Gate) & $1.7 \pm 0.4\%$ & $2.1 \pm 0.9\%$ & $1.8 \pm 1.2\%$ & $1.6 \pm 0.9\%$ & $2.3 \pm 1.1\%$ & $2.8 \pm 1.8\%$ \\
Flat Transformer (Sinusoidal) & $100.0 \pm 0.0\%$ & $100.0 \pm 0.0\%$ & $100.0 \pm 0.0\%$ & $30.8 \pm 7.1\%$ & $17.7 \pm 5.1\%$ & OOM \\
Flat Transformer (RoPE + AttnPool) & $100.0 \pm 0.0\%$ & $100.0 \pm 0.0\%$ & $100.0 \pm 0.0\%$ & $14.1 \pm 4.4\%$ & $3.8 \pm 1.6\%$ & OOM \\
\bottomrule
\end{tabular}
\caption{Zero-shot length extrapolation accuracy on task \textit{intrusion\_\allowbreak fsm} (Part 1).}
\label{tab:intrusion_fsm_pt1}
\end{table}

\begin{table}[htbp]
\centering
\tiny
\setlength{\tabcolsep}{3pt}
\begin{tabular}{p{3.5cm}cccccc}
\toprule
Architecture & $N=65536$ & $N=131072$ & $N=262144$ & $N=524288$ & $N=1048576$ & $N=2097152$ \\
\midrule
GraphTree (Bidirectional: Up + Down) & $100.0 \pm 0.0\%$ & $100.0 \pm 0.0\%$ & $100.0 \pm 0.0\%$ & $100.0 \pm 0.0\%$ & $100.0 \pm 0.0\%$ & $98.8 \pm 2.5\%$ \\
GraphTree (Bottom-Up + Shortcuts) & $100.0 \pm 0.0\%$ & $100.0 \pm 0.0\%$ & $100.0 \pm 0.0\%$ & $100.0 \pm 0.0\%$ & $98.8 \pm 2.5\%$ & $100.0 \pm 0.0\%$ \\
GraphTree (Local Shift, No Expander) & $100.0 \pm 0.0\%$ & $100.0 \pm 0.0\%$ & $100.0 \pm 0.0\%$ & $100.0 \pm 0.0\%$ & $100.0 \pm 0.0\%$ & $100.0 \pm 0.0\%$ \\
Pure Binary Tree-GNN (No Shortcuts) & $100.0 \pm 0.0\%$ & $100.0 \pm 0.0\%$ & $100.0 \pm 0.0\%$ & $100.0 \pm 0.0\%$ & $100.0 \pm 0.0\%$ & $100.0 \pm 0.0\%$ \\
Unbounded Tree-GNN (No Convex Gate) & $1.2 \pm 1.5\%$ & $5.0 \pm 4.7\%$ & $1.2 \pm 2.5\%$ & $1.2 \pm 2.5\%$ & $0.0 \pm 0.0\%$ & $1.2 \pm 2.5\%$ \\
Flat Transformer (Sinusoidal) & OOM & OOM & OOM & OOM & OOM & OOM \\
Flat Transformer (RoPE + AttnPool) & OOM & OOM & OOM & OOM & OOM & OOM \\
\bottomrule
\end{tabular}
\caption{Zero-shot length extrapolation accuracy on task \textit{intrusion\_\allowbreak fsm} (Part 2).}
\label{tab:intrusion_fsm_pt2}
\end{table}

\begin{table}[htbp]
\centering
\tiny
\setlength{\tabcolsep}{3pt}
\begin{tabular}{p{3.5cm}cccccc}
\toprule
Architecture & $N=64$ (ID) & $N=256$ (ID) & $N=512$ (ID) & $N=1024$ (OOD) & $N=4096$ (OOD) & $N=16384$ (OOD) \\
\midrule
GraphTree (Bidirectional: Up + Down) & $84.9 \pm 30.2\%$ & $85.8 \pm 28.3\%$ & $85.5 \pm 29.0\%$ & $87.0 \pm 26.0\%$ & $85.3 \pm 29.4\%$ & $85.6 \pm 28.8\%$ \\
GraphTree (Bottom-Up + Shortcuts) & $100.0 \pm 0.0\%$ & $100.0 \pm 0.0\%$ & $100.0 \pm 0.0\%$ & $100.0 \pm 0.0\%$ & $100.0 \pm 0.0\%$ & $100.0 \pm 0.0\%$ \\
GraphTree (Local Shift, No Expander) & $100.0 \pm 0.0\%$ & $100.0 \pm 0.0\%$ & $100.0 \pm 0.0\%$ & $100.0 \pm 0.0\%$ & $100.0 \pm 0.0\%$ & $100.0 \pm 0.0\%$ \\
Pure Binary Tree-GNN (No Shortcuts) & $100.0 \pm 0.0\%$ & $100.0 \pm 0.0\%$ & $100.0 \pm 0.0\%$ & $100.0 \pm 0.0\%$ & $100.0 \pm 0.0\%$ & $100.0 \pm 0.0\%$ \\
Unbounded Tree-GNN (No Convex Gate) & $1.8 \pm 0.5\%$ & $1.2 \pm 0.5\%$ & $1.8 \pm 0.5\%$ & $1.5 \pm 0.8\%$ & $0.6 \pm 0.6\%$ & $2.5 \pm 1.6\%$ \\
Flat Transformer (Sinusoidal) & $100.0 \pm 0.0\%$ & $100.0 \pm 0.0\%$ & $100.0 \pm 0.0\%$ & $30.4 \pm 15.9\%$ & $13.9 \pm 5.4\%$ & OOM \\
Flat Transformer (RoPE + AttnPool) & $100.0 \pm 0.0\%$ & $100.0 \pm 0.0\%$ & $100.0 \pm 0.0\%$ & $28.9 \pm 13.6\%$ & $4.7 \pm 1.9\%$ & OOM \\
\bottomrule
\end{tabular}
\caption{Zero-shot length extrapolation accuracy on task \textit{code\_\allowbreak scope\_\allowbreak shadow} (Part 1).}
\label{tab:code_scope_shadow_pt1}
\end{table}

\begin{table}[htbp]
\centering
\tiny
\setlength{\tabcolsep}{3pt}
\begin{tabular}{p{3.5cm}cccccc}
\toprule
Architecture & $N=65536$ & $N=131072$ & $N=262144$ & $N=524288$ & $N=1048576$ & $N=2097152$ \\
\midrule
GraphTree (Bidirectional: Up + Down) & $87.5 \pm 23.5\%$ & $85.0 \pm 30.0\%$ & $85.0 \pm 30.0\%$ & $81.2 \pm 34.5\%$ & $87.5 \pm 25.0\%$ & $83.8 \pm 29.5\%$ \\
GraphTree (Bottom-Up + Shortcuts) & $98.8 \pm 2.5\%$ & $100.0 \pm 0.0\%$ & $95.0 \pm 7.3\%$ & $88.8 \pm 17.0\%$ & $91.2 \pm 9.3\%$ & $81.2 \pm 15.8\%$ \\
GraphTree (Local Shift, No Expander) & $100.0 \pm 0.0\%$ & $100.0 \pm 0.0\%$ & $98.8 \pm 2.5\%$ & $98.8 \pm 2.5\%$ & $100.0 \pm 0.0\%$ & $98.8 \pm 2.5\%$ \\
Pure Binary Tree-GNN (No Shortcuts) & $100.0 \pm 0.0\%$ & $100.0 \pm 0.0\%$ & $100.0 \pm 0.0\%$ & $100.0 \pm 0.0\%$ & $100.0 \pm 0.0\%$ & $100.0 \pm 0.0\%$ \\
Unbounded Tree-GNN (No Convex Gate) & $1.2 \pm 1.5\%$ & $1.2 \pm 2.5\%$ & $2.5 \pm 3.1\%$ & $2.5 \pm 3.1\%$ & $0.0 \pm 0.0\%$ & $2.5 \pm 3.1\%$ \\
Flat Transformer (Sinusoidal) & OOM & OOM & OOM & OOM & OOM & OOM \\
Flat Transformer (RoPE + AttnPool) & OOM & OOM & OOM & OOM & OOM & OOM \\
\bottomrule
\end{tabular}
\caption{Zero-shot length extrapolation accuracy on task \textit{code\_\allowbreak scope\_\allowbreak shadow} (Part 2).}
\label{tab:code_scope_shadow_pt2}
\end{table}

\begin{table}[htbp]
\centering
\tiny
\setlength{\tabcolsep}{3pt}
\begin{tabular}{p{3.5cm}cccccc}
\toprule
Architecture & $N=64$ (ID) & $N=256$ (ID) & $N=512$ (ID) & $N=1024$ (OOD) & $N=4096$ (OOD) & $N=16384$ (OOD) \\
\midrule
GraphTree (Bidirectional: Up + Down) & $8.2 \pm 1.9\%$ & $10.7 \pm 2.7\%$ & $10.9 \pm 0.9\%$ & $11.3 \pm 2.1\%$ & $9.1 \pm 0.9\%$ & $12.2 \pm 2.3\%$ \\
GraphTree (Bottom-Up + Shortcuts) & $100.0 \pm 0.0\%$ & $100.0 \pm 0.0\%$ & $100.0 \pm 0.0\%$ & $100.0 \pm 0.0\%$ & $100.0 \pm 0.0\%$ & $100.0 \pm 0.0\%$ \\
GraphTree (Local Shift, No Expander) & $100.0 \pm 0.0\%$ & $100.0 \pm 0.0\%$ & $100.0 \pm 0.0\%$ & $100.0 \pm 0.0\%$ & $100.0 \pm 0.0\%$ & $100.0 \pm 0.0\%$ \\
Pure Binary Tree-GNN (No Shortcuts) & $9.2 \pm 2.8\%$ & $12.2 \pm 3.9\%$ & $10.8 \pm 2.3\%$ & $12.2 \pm 1.7\%$ & $11.4 \pm 2.0\%$ & $11.2 \pm 1.5\%$ \\
Unbounded Tree-GNN (No Convex Gate) & $1.5 \pm 0.8\%$ & $1.2 \pm 0.7\%$ & $2.2 \pm 1.0\%$ & $1.7 \pm 0.6\%$ & $1.9 \pm 0.8\%$ & $1.6 \pm 1.0\%$ \\
Flat Transformer (Sinusoidal) & $9.1 \pm 1.1\%$ & $5.0 \pm 2.0\%$ & $4.8 \pm 0.6\%$ & $2.9 \pm 0.7\%$ & $1.9 \pm 0.6\%$ & OOM \\
Flat Transformer (RoPE + AttnPool) & $8.7 \pm 1.1\%$ & $4.9 \pm 0.5\%$ & $3.3 \pm 0.5\%$ & $2.9 \pm 0.5\%$ & $2.3 \pm 1.9\%$ & OOM \\
\bottomrule
\end{tabular}
\caption{Zero-shot length extrapolation accuracy on task \textit{multi\_\allowbreak needle\_\allowbreak rag} (Part 1).}
\label{tab:multi_needle_rag_pt1}
\end{table}

\begin{table}[htbp]
\centering
\tiny
\setlength{\tabcolsep}{3pt}
\begin{tabular}{p{3.5cm}cccccc}
\toprule
Architecture & $N=65536$ & $N=131072$ & $N=262144$ & $N=524288$ & $N=1048576$ & $N=2097152$ \\
\midrule
GraphTree (Bidirectional: Up + Down) & $11.9 \pm 4.6\%$ & $13.8 \pm 6.1\%$ & $11.2 \pm 13.3\%$ & $6.2 \pm 6.8\%$ & $8.8 \pm 6.4\%$ & $8.8 \pm 7.5\%$ \\
GraphTree (Bottom-Up + Shortcuts) & $100.0 \pm 0.0\%$ & $100.0 \pm 0.0\%$ & $100.0 \pm 0.0\%$ & $100.0 \pm 0.0\%$ & $100.0 \pm 0.0\%$ & $100.0 \pm 0.0\%$ \\
GraphTree (Local Shift, No Expander) & $100.0 \pm 0.0\%$ & $100.0 \pm 0.0\%$ & $100.0 \pm 0.0\%$ & $100.0 \pm 0.0\%$ & $100.0 \pm 0.0\%$ & $100.0 \pm 0.0\%$ \\
Pure Binary Tree-GNN (No Shortcuts) & $12.5 \pm 7.1\%$ & $12.5 \pm 5.6\%$ & $16.2 \pm 12.9\%$ & $8.8 \pm 6.4\%$ & $13.8 \pm 4.7\%$ & $16.2 \pm 7.5\%$ \\
Unbounded Tree-GNN (No Convex Gate) & $1.2 \pm 2.5\%$ & $1.2 \pm 2.5\%$ & $1.2 \pm 2.5\%$ & $1.2 \pm 2.5\%$ & $2.5 \pm 3.1\%$ & $1.2 \pm 2.5\%$ \\
Flat Transformer (Sinusoidal) & OOM & OOM & OOM & OOM & OOM & OOM \\
Flat Transformer (RoPE + AttnPool) & OOM & OOM & OOM & OOM & OOM & OOM \\
\bottomrule
\end{tabular}
\caption{Zero-shot length extrapolation accuracy on task \textit{multi\_\allowbreak needle\_\allowbreak rag} (Part 2).}
\label{tab:multi_needle_rag_pt2}
\end{table}

\begin{table}[htbp]
\centering
\tiny
\setlength{\tabcolsep}{3pt}
\begin{tabular}{p{3.5cm}cccccc}
\toprule
Architecture & $N=64$ (ID) & $N=256$ (ID) & $N=512$ (ID) & $N=1024$ (OOD) & $N=4096$ (OOD) & $N=16384$ (OOD) \\
\midrule
GraphTree (Bidirectional: Up + Down) & $96.7 \pm 6.7\%$ & $96.9 \pm 6.2\%$ & $96.7 \pm 6.7\%$ & $96.6 \pm 6.9\%$ & $97.0 \pm 5.9\%$ & $96.9 \pm 6.2\%$ \\
GraphTree (Bottom-Up + Shortcuts) & $100.0 \pm 0.0\%$ & $100.0 \pm 0.0\%$ & $100.0 \pm 0.0\%$ & $100.0 \pm 0.0\%$ & $100.0 \pm 0.0\%$ & $100.0 \pm 0.0\%$ \\
GraphTree (Local Shift, No Expander) & $100.0 \pm 0.0\%$ & $100.0 \pm 0.0\%$ & $100.0 \pm 0.0\%$ & $100.0 \pm 0.0\%$ & $100.0 \pm 0.0\%$ & $100.0 \pm 0.0\%$ \\
Pure Binary Tree-GNN (No Shortcuts) & $100.0 \pm 0.0\%$ & $100.0 \pm 0.0\%$ & $100.0 \pm 0.0\%$ & $100.0 \pm 0.0\%$ & $100.0 \pm 0.0\%$ & $100.0 \pm 0.0\%$ \\
Unbounded Tree-GNN (No Convex Gate) & $11.9 \pm 1.8\%$ & $14.0 \pm 1.8\%$ & $15.3 \pm 2.1\%$ & $10.4 \pm 1.8\%$ & $11.7 \pm 2.6\%$ & $8.1 \pm 3.8\%$ \\
Flat Transformer (Sinusoidal) & $100.0 \pm 0.0\%$ & $100.0 \pm 0.0\%$ & $100.0 \pm 0.0\%$ & $99.0 \pm 1.3\%$ & $51.1 \pm 13.0\%$ & OOM \\
Flat Transformer (RoPE + AttnPool) & $100.0 \pm 0.0\%$ & $100.0 \pm 0.0\%$ & $100.0 \pm 0.0\%$ & $57.0 \pm 8.5\%$ & $11.7 \pm 2.9\%$ & OOM \\
\bottomrule
\end{tabular}
\caption{Zero-shot length extrapolation accuracy on task \textit{xml\_\allowbreak bracket\_\allowbreak check} (Part 1).}
\label{tab:xml_bracket_check_pt1}
\end{table}

\begin{table}[htbp]
\centering
\tiny
\setlength{\tabcolsep}{3pt}
\begin{tabular}{p{3.5cm}cccccc}
\toprule
Architecture & $N=65536$ & $N=131072$ & $N=262144$ & $N=524288$ & $N=1048576$ & $N=2097152$ \\
\midrule
GraphTree (Bidirectional: Up + Down) & $96.9 \pm 6.2\%$ & $97.5 \pm 5.0\%$ & $98.8 \pm 2.5\%$ & $95.0 \pm 10.0\%$ & $95.0 \pm 10.0\%$ & $100.0 \pm 0.0\%$ \\
GraphTree (Bottom-Up + Shortcuts) & $100.0 \pm 0.0\%$ & $100.0 \pm 0.0\%$ & $100.0 \pm 0.0\%$ & $98.8 \pm 2.5\%$ & $100.0 \pm 0.0\%$ & $100.0 \pm 0.0\%$ \\
GraphTree (Local Shift, No Expander) & $100.0 \pm 0.0\%$ & $100.0 \pm 0.0\%$ & $100.0 \pm 0.0\%$ & $100.0 \pm 0.0\%$ & $100.0 \pm 0.0\%$ & $100.0 \pm 0.0\%$ \\
Pure Binary Tree-GNN (No Shortcuts) & $100.0 \pm 0.0\%$ & $100.0 \pm 0.0\%$ & $100.0 \pm 0.0\%$ & $100.0 \pm 0.0\%$ & $100.0 \pm 0.0\%$ & $100.0 \pm 0.0\%$ \\
Unbounded Tree-GNN (No Convex Gate) & $11.2 \pm 7.5\%$ & $17.5 \pm 7.3\%$ & $13.8 \pm 10.0\%$ & $10.0 \pm 5.0\%$ & $11.2 \pm 4.7\%$ & $17.5 \pm 6.1\%$ \\
Flat Transformer (Sinusoidal) & OOM & OOM & OOM & OOM & OOM & OOM \\
Flat Transformer (RoPE + AttnPool) & OOM & OOM & OOM & OOM & OOM & OOM \\
\bottomrule
\end{tabular}
\caption{Zero-shot length extrapolation accuracy on task \textit{xml\_\allowbreak bracket\_\allowbreak check} (Part 2).}
\label{tab:xml_bracket_check_pt2}
\end{table}

\begin{table}[htbp]
\centering
\tiny
\setlength{\tabcolsep}{3pt}
\begin{tabular}{p{3.5cm}cccccc}
\toprule
Architecture & $N=64$ (ID) & $N=256$ (ID) & $N=512$ (ID) & $N=1024$ (OOD) & $N=4096$ (OOD) & $N=16384$ (OOD) \\
\midrule
GraphTree (Bidirectional: Up + Down) & $100.0 \pm 0.0\%$ & $100.0 \pm 0.0\%$ & $100.0 \pm 0.0\%$ & $100.0 \pm 0.0\%$ & $100.0 \pm 0.0\%$ & $99.4 \pm 1.2\%$ \\
GraphTree (Bottom-Up + Shortcuts) & $100.0 \pm 0.0\%$ & $100.0 \pm 0.0\%$ & $100.0 \pm 0.0\%$ & $100.0 \pm 0.0\%$ & $100.0 \pm 0.0\%$ & $100.0 \pm 0.0\%$ \\
GraphTree (Local Shift, No Expander) & $100.0 \pm 0.0\%$ & $100.0 \pm 0.0\%$ & $100.0 \pm 0.0\%$ & $100.0 \pm 0.0\%$ & $100.0 \pm 0.0\%$ & $100.0 \pm 0.0\%$ \\
Pure Binary Tree-GNN (No Shortcuts) & $98.7 \pm 1.1\%$ & $98.0 \pm 1.0\%$ & $99.1 \pm 1.0\%$ & $99.0 \pm 1.1\%$ & $97.7 \pm 2.6\%$ & $91.2 \pm 10.9\%$ \\
Unbounded Tree-GNN (No Convex Gate) & $3.0 \pm 3.8\%$ & $2.9 \pm 2.3\%$ & $2.4 \pm 1.1\%$ & $2.9 \pm 2.3\%$ & $2.8 \pm 1.5\%$ & $1.9 \pm 2.5\%$ \\
Flat Transformer (Sinusoidal) & $17.9 \pm 2.8\%$ & $3.5 \pm 1.0\%$ & $3.6 \pm 1.7\%$ & $1.9 \pm 0.6\%$ & $1.4 \pm 0.6\%$ & OOM \\
Flat Transformer (RoPE + AttnPool) & $15.5 \pm 2.0\%$ & $4.4 \pm 1.0\%$ & $3.0 \pm 1.0\%$ & $1.9 \pm 1.1\%$ & $0.9 \pm 0.8\%$ & OOM \\
\bottomrule
\end{tabular}
\caption{Zero-shot length extrapolation accuracy on task \textit{adversarial\_\allowbreak decoys} (Part 1).}
\label{tab:adversarial_decoys_pt1}
\end{table}

\begin{table}[htbp]
\centering
\tiny
\setlength{\tabcolsep}{3pt}
\begin{tabular}{p{3.5cm}cccccc}
\toprule
Architecture & $N=65536$ & $N=131072$ & $N=262144$ & $N=524288$ & $N=1048576$ & $N=2097152$ \\
\midrule
GraphTree (Bidirectional: Up + Down) & $92.5 \pm 15.0\%$ & $83.8 \pm 32.5\%$ & $82.5 \pm 35.0\%$ & $81.2 \pm 37.5\%$ & $81.2 \pm 37.5\%$ & $80.0 \pm 40.0\%$ \\
GraphTree (Bottom-Up + Shortcuts) & $99.4 \pm 1.2\%$ & $100.0 \pm 0.0\%$ & $100.0 \pm 0.0\%$ & $97.5 \pm 3.1\%$ & $95.0 \pm 7.3\%$ & $96.2 \pm 5.0\%$ \\
GraphTree (Local Shift, No Expander) & $100.0 \pm 0.0\%$ & $100.0 \pm 0.0\%$ & $100.0 \pm 0.0\%$ & $100.0 \pm 0.0\%$ & $97.5 \pm 5.0\%$ & $98.8 \pm 2.5\%$ \\
Pure Binary Tree-GNN (No Shortcuts) & $86.2 \pm 13.2\%$ & $86.2 \pm 18.7\%$ & $75.0 \pm 23.1\%$ & $81.2 \pm 17.2\%$ & $70.0 \pm 18.3\%$ & $66.2 \pm 25.2\%$ \\
Unbounded Tree-GNN (No Convex Gate) & $1.9 \pm 3.8\%$ & $2.5 \pm 3.1\%$ & $1.2 \pm 2.5\%$ & $1.2 \pm 2.5\%$ & $1.2 \pm 2.5\%$ & $2.5 \pm 3.1\%$ \\
Flat Transformer (Sinusoidal) & OOM & OOM & OOM & OOM & OOM & OOM \\
Flat Transformer (RoPE + AttnPool) & OOM & OOM & OOM & OOM & OOM & OOM \\
\bottomrule
\end{tabular}
\caption{Zero-shot length extrapolation accuracy on task \textit{adversarial\_\allowbreak decoys} (Part 2).}
\label{tab:adversarial_decoys_pt2}
\end{table}

\end{document}